\documentclass{article}

\usepackage{amsmath,amssymb,amsfonts,amsthm}
\usepackage{microtype}
\usepackage{graphicx}
\usepackage{booktabs}
\usepackage{nicefrac}
\usepackage{subcaption}

\usepackage{pifont}
\newcommand{\cmark}{\ding{51}}%
\newcommand{\xmark}{\ding{55}}%

\usepackage{color}
\usepackage{ulem}
\usepackage{hyperref}

\usepackage[accepted]{icml2021}

\icmltitlerunning{Lower-Bounded Proper Losses for Weakly Supervised Classification}

\newtheorem{theorem}{Theorem}
\newtheorem{lemma}[theorem]{Lemma}
\newtheorem{corollary}[theorem]{Corollary}
\newtheorem{proposition}[theorem]{Proposition}
\theoremstyle{definition}
\newtheorem{example}[theorem]{Example}
\newtheorem{definition}[theorem]{Definition}

\DeclareMathOperator*{\argmin}{arg\,min}
\DeclareMathOperator*{\argmax}{arg\,max}
\DeclareMathOperator{\aff}{aff}
\DeclareMathOperator{\interior}{int}
\DeclareMathOperator{\coker}{coker}

\newcommand{\subgrad}{\underline{\nabla}}
\newcommand{\dataspace}{\mathcal{X}}
\newcommand{\truelabels}{\mathcal{Z}}
\newcommand{\weaklabels}{\mathcal{Y}}

\newcommand{\domain}{\mathcal{C}}
\newcommand{\setD}{\mathcal{D}}

\newcommand{\numtls}{c}
\newcommand{\numwls}{{c_\mathrm{W}}}

\newcommand{\distr}{\mathcal{P}}
\newcommand{\tldistr}{\distr(\truelabels)}
\newcommand{\wldistr}{\distr(\weaklabels)}

\newcommand{\idx}{I_\mathcal{X}}

\newcommand{\idz}{I_\mathcal{Z}}

\renewcommand{\vec}{\boldsymbol}
\newcommand{\one}{\vec{1}}
\newcommand{\onex}{\one_\mathcal{X}}
\newcommand{\oney}{\one_\mathcal{Y}}
\newcommand{\onez}{\one_\mathcal{Z}}
\newcommand{\onexperp}{\onex^\perp}

\newcommand{\onezperp}{\onez^\perp}

\newcommand{\xx}{\vec{x}}
\newcommand{\yy}{\vec{y}}

\newcommand{\pp}{\vec{p}}
\newcommand{\qq}{\vec{q}}
\newcommand{\vv}{\vec{v}}

\newcommand{\expect}{\mathbb{E}}
\newcommand{\real}{\mathbb{R}}
\newcommand{\extreal}{\overline{\real}}

\newcommand{\inner}[2]{\langle #1, #2 \rangle}
\newcommand{\card}[1]{\left\lvert #1 \right\rvert}

\newcommand{\tproper}{$T$-proper\ }
\newcommand{\T}{^\mathsf{T}}

\begin{document}

\twocolumn[
\icmltitle{Lower-Bounded Proper Losses for Weakly Supervised Classification}

\begin{icmlauthorlist}
\icmlauthor{Shuhei M.~Yoshida}{nec,aip}
\icmlauthor{Takashi Takenouchi}{aip,fuh}
\icmlauthor{Masashi Sugiyama}{aip,ut}
\end{icmlauthorlist}

\icmlaffiliation{nec}{Biometrics Research Laboratories, NEC Corporation, Kawasaki, Kanagawa, Japan}
\icmlaffiliation{aip}{RIKEN Center for Advanced Intelligence Project, Chuo-ku, Tokyo, Japan}
\icmlaffiliation{fuh}{National Graduate Institute for Policy Studies, Minato-ku, Tokyo, Japan}
\icmlaffiliation{ut}{Department of Complexity Science and Engineering, The University of Tokyo, Kashiwa, Chiba, Japan}

\icmlcorrespondingauthor{Shuhei M.~Yoshida}{s\_m\_yoshida@nec.com}

\icmlkeywords{Weakly supervised learning, proper loss, regularization}

\vskip 0.3in
]

\printAffiliationsAndNotice{}

\begin{abstract}
This paper discusses the problem of weakly supervised classification, in which instances are given weak labels that are produced by some label-corruption process. The goal is to derive conditions under which loss functions for weak-label learning are proper and lower-bounded---two essential requirements for the losses used in class-probability estimation. To this end, we derive a representation theorem for proper losses in supervised learning, which dualizes the Savage representation. We use this theorem to characterize proper weak-label losses and find a condition for them to be lower-bounded. From these theoretical findings, we derive a novel regularization scheme called generalized logit squeezing, which makes any proper weak-label loss bounded from below, without losing properness. Furthermore, we experimentally demonstrate the effectiveness of our proposed approach, as compared to improper or unbounded losses. The results highlight the importance of properness and lower-boundedness.
\end{abstract}

\section{Introduction}

Recent machine learning techniques have achieved state-of-the-art performance on many prediction tasks,
but they usually require massive training data with clean annotations.
One approach to reduce the costs of data preparation is so-called weakly supervised learning: each instance is annotated with a weak label that is cheaper to obtain but less informative than a true label.
For classification, many types of weak supervision have been proposed.
For example, in learning with noisy labels~\citep{angluin.laird1988,natarajan.etal2013,patrini.etal2017}, one observes an instance with a label that may be corrupted.
Positive-unlabeled (PU) learning of binary classification uses positive and unlabeled data, but not labeled negative data~\citep{elkan.noto2008,duplessis.etal2015}.
Another example is learning from partial labels, which are collections of candidate labels among which only one is true~\citep{cour.etal2011}.
Many of these approaches are understood as learning from weak labels that are produced by label-corruption processes,
and some authors have taken unified approaches to tackle these problems~\citep{vanrooyen.williamson2018,zhang.etal2019d}.

A fundamental theoretical question is under what conditions learning from weak labels is possible.
To address this question, analysis of loss functions plays a central role.
Among loss functions, proper losses are a particularly important class of losses that can correctly estimate class posterior probabilities~\citep{winkler.murphy1968,buja.etal2005,gneiting.raftery2007}.
Two major classes of proper weak-label losses have been proposed in the literature.
One class derives from unbiased risk estimation, or backward loss correction~\citep{patrini.etal2017},
in which a label-corruption process is inverted to estimate an expected risk with respect to the distribution of true labels.
This approach has been taken, for example, in partial-label learning~\citep{cid-sueiro2012}, noisy-label learning~\citep{natarajan.etal2013,patrini.etal2017}, PU learning~\citep{duplessis.etal2015}, and complementary-label learning~\citep{ishida.etal2019}.
For a general label-corruption process, a recent work showed how to construct a proper weak-label loss from a loss for supervised learning~\citep{vanrooyen.williamson2018}.
The other class of losses follows from forward loss correction~\citep{patrini.etal2017}, in which proper loss functions are used for estimating the posterior distribution of weak labels.
This approach has been applied to noisy-label learning~\citep{patrini.etal2017} and complementary-label learning~\citep{yu.etal2018}.
Moreover, \citet{zhang.etal2019d} applied a forward-corrected loss to more general problems of learning from weak labels, although their discussion focused on the negative log-likelihood loss.

In addition to properness, lower-boundedness is another important requirement for loss functions so that learning can succeed.
Losses that are not bounded from below are problematic, as they cause the objective to diverge to negative infinity, especially when using complex models like deep neural networks~\citep{kiryo.etal2017}.
Forward-corrected losses are known to be proper and lower-bounded~\citep{patrini.etal2017,yu.etal2018}.
On the other hand, backward-corrected losses are generally not guaranteed to be bounded from below~\citep{natarajan.etal2013,cid-sueiro.etal2014,duplessis.etal2015,kiryo.etal2017,patrini.etal2017,ishida.etal2019,vanrooyen.williamson2018}.
From a practical viewpoint, implementation tricks proposed by \citet{kiryo.etal2017} cause a training objective to be positive and work reasonably well, but they also result in an improper loss.
Those tricks have also been applied to complementary-label learning~\citep{ishida.etal2019} and unlabeled-unlabeled learning~\citep{lu.etal2020a}.
\citet{chou.etal2020} proposed a novel class of surrogate losses that are bounded from below, but these losses are not guaranteed to be proper.
To the best of our knowledge, conditions under which proper weak-label losses are bounded from below have yet to be addressed.

\paragraph{Our contributions}
This paper discusses proper losses for weakly supervised learning of class posterior probability estimation.
In particular, we obtain conditions under which proper weak-label losses are bounded from below.
To do so, we derive the dual representation of proper losses for supervised learning.
This representation is a dualized version of the Savage representation~\citep{savage1971,cid-sueiro.etal1999,gneiting.raftery2007}, which characterizes a proper loss in terms of a Bayes risk.
By using a theorem that we obtain, we characterize proper weak-label losses and derive a sufficient condition under which the resulting losses are bounded from below.
The derived condition is not necessary but covers a large class of losses that are parameterized by convex functions constrained by a single inequality.

From these results, we derive a novel regularization scheme called generalized logit squeezing (gLS), which makes any proper weak-label loss bounded from below, without losing its properness.
We also experimentally demonstrate the effectiveness of our proposed approach as compared to unbounded or improper losses.
We show that gLS yields superior or competitive results as compared to baseline methods, regardless of the precise values of the hyperparameters that are specific to gLS, as long as those parameters are in the regime in which gLS gives $T$-proper and bounded losses.

\section{Formulation}

In this section, we introduce notations and basic notions, which we adopted from previous studies~\citep{winkler.murphy1968,buja.etal2005,gneiting.raftery2007,cid-sueiro2012,vanrooyen.williamson2018}. 
We begin by summarizing the mathematical notations in Section~\ref{sec:notation}.
Then, the two key notions of weak labels and proper losses are described in Sections~\ref{sec:weaklabels} and~\ref{sec:Tproper}, respectively.

\subsection{Notations}\label{sec:notation}

Boldface and calligraphic letters respectively denote vectors and sets.
The sets of real numbers and extended real numbers are denoted by $\real$ and $\extreal\equiv \real\cup\{-\infty,\infty\}$, respectively. Let $\mathcal{X}$ be a discrete set and $\card{\mathcal{X}}$ be its cardinality.
The set $\real^{\mathcal{X}}$ is the $\card{\mathcal{X}}$-dimensional vector space whose dimensions are indexed with $x\in\mathcal{X}$.
A matrix $\idx$ is the identity matrix on $\real^\mathcal{X}$,
$\onex$ is a vector in $\real^\mathcal{X}$ such that $(\onex)_x=1$ for all $x\in \mathcal{X}$,
and $\onexperp$ is the orthogonal complement of $\onex$. The set of probability distributions over $\mathcal{X}$ is identified with the probability simplex $\mathcal{P}(\mathcal{X}) \equiv
        \{ \pp\in\real^{\mathcal{X}} \vert
            \sum_{x\in\mathcal{X}} p_x = 1, 
            p_x \geq 0 \text{ for all } x \in \mathcal{X}\}$.

The theory of convex functions has offered useful tools for analyzing proper losses~\citep{gneiting.raftery2007,dawid2007}. A function $f:\domain\to\real$ is convex if
$ 
    f((1-\lambda)\xx_0 + \lambda \xx_1) \leq (1-\lambda) f(\xx_0) + \lambda f(\xx_1)
    \label{eq:def-convex}
$ 
for all $\lambda \in (0,1)$ and $\xx_0, \xx_1 \in \domain$.
It is strictly convex if the equality holds only when $\xx_0 = \xx_1$.
A convex function $f$ is said to be closed if its epigraph $\{ (\xx, t) \in \domain\times\real \vert t \geq f(\xx) \}$ is a closed set.
A vector $\subgrad f(\xx)$ is a subgradient of $f$ at a point $\xx\in \domain$ if it satisfies $f(\yy) \geq f(\xx) + \inner{\subgrad f(\xx)}{\yy-\xx}$ for all $\yy \in \domain$.
In general, subgradients may not be unique at a given point.
The set of all the subgradients of $f$ at $\xx\in \domain$ is called the subdifferential and is denoted by $\partial f(\xx)$.
The convex conjugate of a convex function $f: \domain\to\real$ is denoted by $f^\ast$ and is defined as $f^\ast(\vv) = \sup_{x\in\domain} \left[ \inner{\vv}{\xx} - f(\xx) \right]$.

\subsection{Weak Labels in Classification Learning}\label{sec:weaklabels}

Let $\dataspace$ be a space of instances, $\truelabels=\{ z_1, z_2,\dots, z_\numtls \}$ be a set of true (latent) labels, and $\weaklabels=\{ y_1, y_2, \dots, y_\numwls \}$ be a set of weak (observed) labels.
In weakly supervised learning of classification, an algorithm is given a training set sampled from $\dataspace\times\weaklabels$ in accordance with an unknown data distribution,
and it learns to classify an instance $x\in\dataspace$ into a true class $z\in\truelabels$.
The true labels for training instances are not available to the learner.

We focus on a setting in which weak labels are characterized by a conditional distribution $p(y \vert x, z)$, or a label transition matrix $T (x)$, whose matrix element $T_{yz}(x)$ is $p(y\vert x, z)$.
In this paper, we assume that (a) $T(x) \equiv T$, which means that weak labels are independent of input data $x$,
and that (b) $T$ has a left inverse $R$ such that $RT = \idz$.
In principle, Assumption~(a) can be lifted by replacing $T$ with $T(x)$ and $R$ with $R(x)$ in the following analysis, even though such scenarios are more challenging to deal with in practice, because they require knowing $T(x)$ for all $x\in\dataspace$. Assumption~(b) requires that weak labels be informative enough for a learner to infer a distribution over the true labels.
Concretely, we can reconstruct true-label posterior probabilities from weak-label posterior probabilities by using the following identity:
\begin{align}
    p(z \vert x) 
    = \sum_{z' \in \truelabels} (RT)_{zz'}
        p(z' \vert x)
    = \sum_{y\in \weaklabels} R_{zy} p(y \vert x).
\end{align}
A label transition matrix $T$ satisfying Assumption~(b) is said to be reconstructible, and $R$ is called a reconstruction matrix of $T$.
In particular, $T$ is reconstructible only if $|\truelabels| \leq |\weaklabels|$.

Solving weakly supervised classification always requires some assumption like Assumption~(b) that constrains the form of $T$.
See Appendix~\ref{apdx:left-invertibility} for a comparison of Assumption~(b) with other assumptions that have been made in previous works.

The following are illustrative examples with $\truelabels=\{ z_1 , z_2, z_3 \}$.

\begin{example}[Learning with label noise, \citet{natarajan.etal2013}]
If instances are equipped with noisy labels, then the weak-label set $\weaklabels$ is identical to $\truelabels$. For a three-class setting with symmetric noise, $T$ is
\begin{align}
    T &= \begin{pmatrix}
    1-p & p/2 & p/2 \\
    p/2 & 1-p & p/2 \\
    p/2 & p/2 & 1-p
    \end{pmatrix}, 
\end{align}
and its reconstruction matrix is
\begin{align}
    R &= \frac{1}{2-3p}\begin{pmatrix}
    2-p & -p & -p \\
    -p & 2-p & -p \\
    -p & -p & 2-p
    \end{pmatrix},
\end{align}
where $p\in(0,1)$ is the mislabeled probability. Note that $T$ is not reconstructible if $p=\frac{2}{3}$, in which case the weak labels become independent of the true labels.
\end{example}

\begin{example}[Partial labels, \citet{cour.etal2011}]\label{ex:partial-label}
Consider three-class classification with $\weaklabels = \{ (1,0,0), (0,1,0), \linebreak (0,0,1), (1,1,0), (1,0,1), (0,1,1), (1,1,1) \}
$.
A label $y\in\weaklabels$ is called a partial label.
For example, $(1,1,0)$ indicates that the true label is either $z_1$ or $z_2$, but not $z_3$.
In a scenario in which a spurious label is added with probability $p$, the label transition matrix $T$ is
\begin{align}
    T=\begin{pmatrix}
    T_1\T & T_2\T & T_3\T
    \end{pmatrix}\T,
\end{align}
where
\begin{align}
    T_1&= (1-p)^2 I_3, \\
    T_2&=\begin{pmatrix}
    (1-p)p & (1-p)p & 0 \\
    (1-p)p & 0 & (1-p)p \\
    0 & (1-p)p & (1-p)p \\
    \end{pmatrix}, \\
    T_3&=\begin{pmatrix}
    p^2 &
    p^2 &
    p^2 
    \end{pmatrix}.
\end{align}
This $T$ is left-invertible unless $p=1$. The left-inverse is not unique.
\end{example}

So far, we have assumed that there is only one weak-label set $\weaklabels$ and a label transition matrix $T$, and the arguments in the rest of this paper are made for such a scenario.
Note, however, that the arguments here can also be applied to scenarios in which two or more data sources with different noise characteristics are available.
Importantly, this can be done without changing any formal aspect of our theory.
See Appendix~\ref{apdx:multisource} for the details of this point.

\subsection{Proper Losses for Weak-Label Learning}\label{sec:Tproper}

A common strategy for classification is to estimate the class posterior probabilities. To this end, an expected loss should preferably be minimized when an estimator gives the true posterior probabilities:
\begin{align}
    \expect_{(x,z)\sim p(x,z)}
        [ l(q( z \vert x), z) ]
    \geq \expect_{(x,z)\sim p(x,z)}
        [ l(p(z \vert x), z) ],
    \label{eq:def-proper}
\end{align}
where $p(x,z)\in\distr(\dataspace\times\truelabels)$ is a sample distribution, $q(z\vert x) \in \tldistr$ denotes the estimated posterior probabilities for a given instance $x\in\dataspace$, and $l:\tldistr\times\truelabels \to \extreal$ is a loss function.
Because the inequalities at different points in $\dataspace$ are mutually independent, we focus on the conditional risk at a fixed $x$, omit the conditioning variable $x$, and simply use a vector notation like $\pp\in\tldistr$ for the class posterior probabilities in the rest of the paper.
Loss functions satisfying Eq.~\eqref{eq:def-proper} are said to be proper~\citep{winkler.murphy1968}.
A loss function is said to be strictly proper when the equality in Eq.~\eqref{eq:def-proper} holds only if $\pp=\qq$~\citep{gneiting.raftery2007}.
Strict properness is often more desirable than properness itself, because it leads to a Fisher-consistent estimator  $\argmin_{\qq} \expect_{z\sim \pp}[l(\qq, z)]$ for the class posterior probabilities.
It also guarantees that the minima of the empirical and expected losses are unique, which thereby renders the loss minimization problem well-posed.

In weak-label learning, we use a loss function defined on a pair of predicted posterior probabilities $\qq\in\tldistr$ and a weak label $y\in\weaklabels$; we refer to this function as a weak-label loss.
The notion of properness can be extended to weak-label losses~\citep{cid-sueiro2012}.
\begin{definition}
Let $T$ be a label transition matrix. A weak-label loss $l_\mathrm{W}: \tldistr\times\weaklabels \to \extreal$
    is called $T$-proper if, for all $\pp$ and $\qq$ in $\tldistr$,
    \begin{align}
        \expect_{y\sim T\pp} \left[
            l_\mathrm{W}(\qq, y)
        \right]
        \geq \expect_{y\sim T\pp} \left[
            l_\mathrm{W}(\pp, y)
        \right],
        \label{eq:def-T-proper}
    \end{align}
where the vector $T\pp$ is a point in the probability simplex $\wldistr$ and represents a probability distribution over     $\weaklabels$.
The weak-label loss is said to be strictly proper when the equality in Eq.~\eqref{eq:def-T-proper} holds only if $\pp=\qq$.
\end{definition}

\section{Dual Representation of Proper Losses}\label{sec:dual-rep}

In this section, we derive a representation of proper loss functions for supervised learning, which we call a dual representation.
It is closely related to the so-called Savage representation~\citep{savage1971,gneiting.raftery2007}.
The Savage representation expresses a proper loss in terms of its Bayes risk, whereas our representation uses a convex function that is related to the convex conjugate of the Bayes risk.
The dual representation will be useful for our later discussion of the lower-unboundedness of proper weak-label losses.

We start by reviewing the Savage representation,
which requires a mild regularity condition~\citep{gneiting.raftery2007}.
In general, losses can be positive infinity for some $(\qq,z)\in \tldistr\times\truelabels$.
A loss function is said to be regular if it is finite for any $(\qq,z)\in \tldistr\times\truelabels$ except possibly that $l(\qq,z)=\infty$ when $q_z = 0$.
Regular proper losses for class posterior probability estimation are known to have the following representation~\citep{cid-sueiro.etal1999,gneiting.raftery2007}.

\begin{theorem}[Savage representation]\label{thm:savage}
A regular loss function $l:\tldistr\times\truelabels \to \extreal$ is (strictly) proper if and only if there exists a closed (strictly) convex function $S:\tldistr \to \real$ such that for $\qq\in\tldistr$ and $z\in\truelabels$,
\begin{align}
    l(\qq, z) = 
        - \left[ \subgrad S(\qq) \right]_z
        + \inner{\qq}{\subgrad S(\qq)} - S(\qq),
    \label{eq:savage}
\end{align}
where $\subgrad S(\qq) \in \extreal^{\truelabels}$ is a subgradient of $S$ at a point $\qq\in\tldistr$.
\end{theorem}

By using the definition of the subgradient, we can easily verify that the convex function $S$ in the theorem is the negative Bayes risk; that is,
\begin{align}
    S(\pp) = - \min_{\qq\in\tldistr} \expect_{z\sim\pp} \left[
        l(\qq, z)
    \right]
    \equiv -\underline{L}(\pp),
\end{align}
where $\underline{L}(\pp)$ is the Bayes risk.
Thus, Theorem~\ref{thm:savage} shows that a proper loss function is determined by its Bayes risk, up to the choice of $\subgrad S(\qq) \in \partial S(\qq)$ at points where $S$ is not differentiable~\citep{williamson.etal2016}.

Importantly, the sum of the second and third terms in Eq.~\eqref{eq:savage} is the convex conjugate $S^\ast(\subgrad S(\qq))$ of $S$~\cite{reid.etal2015}.
This fact leads to the ``dual'' of the Savage representation.
For a closed convex function $F$ whose domain is a convex subset $\domain$ of $\onezperp$, we define a function $\lambda_F: \domain\times\truelabels \to \real$ as
\begin{align}
    \lambda_F(\vv, z) = - v_z + F(\vv).
\end{align}
The following theorem shows that under a certain condition on $F$, $\lambda_F$ is essentially a proper loss for which $\domain$ parameterizes the probability simplex $\tldistr$.

\begin{theorem}\label{thm:dual-rep}
Let $l:\tldistr\times\truelabels \to \extreal$ be a regular loss. Then, it is proper if and only if there exists a closed convex function $F: \domain\subset\onezperp \to \real$ that satisfies the following conditions:
    \begin{enumerate}
        \item 
        $F(\vv) - \max_{z\in\truelabels} v_z$ is bounded from below.
        \item With $F^\ast(\pp)$ the convex conjugate of $F(\vv)$, it holds that
        $l(\qq, z) = \lambda_F(\subgrad F^\ast(\qq), z)$, where
        $\subgrad F^\ast(\pp)$ is an appropriately chosen subgradient function.
    \end{enumerate}
Furthermore, $F^\ast(\pp)$ at a point $\pp\in\tldistr$ is a negative Bayes risk for this loss.
\end{theorem}

A full proof of this theorem is presented in Appendix~\ref{prf:dual-rep}.
In Appendix~\ref{apdx:strictly-proper}, we also derive conditions on $F$ under which the associated proper loss is strictly proper; however, we do not use them in the following discussion.
Theorem~\ref{thm:dual-rep} elucidates that $F$ in the proved representation is closely related to the convex conjugate of the negative Bayes risk $-\underline{L}$.
Therefore, in the rest of the paper, the representation of a proper loss given in Condition~2 is called the dual representation.

Here, we contrast our Theorem~\ref{thm:dual-rep} with related results.
Indeed, a representation of proper losses that uses $\lambda_F(\vv,z)$ is not new.
\citet{reid.etal2015} showed that proper losses can be written with $\underline{L}^\ast$. 
\citet{vanrooyen.williamson2018} also showed with different proof techniques that any proper loss has the form of Condition~2 in Theorem~\ref{thm:dual-rep}.
In a more general context, \citet{nowak-vila.etal2019} and \citet{blondel.etal2020} discussed loss functions for structured prediction and arrived at the same representation.
There is another line of research on the related notions of matching losses~\cite{kivinen.warmuth1997} and the Bregman divergence~\cite{bregman1967,banerjee.etal2005}, which are the special case of proper losses that have strictly convex and continuously differentiable Bayes risks.
In particular, \citet{amid.etal2019} proved that matching losses have the dual representation.
However, none of those previous studies obtained Condition~1 in Theorem~\ref{thm:dual-rep}, and therefore, they only succeeded in proving the necessity of the dual representation.
In contrast, Theorem~\ref{thm:dual-rep} gives necessary and sufficient conditions for a loss to be proper, which is made possible by constraining the convex functions by Condition~1.
The theorem is also applicable to general proper losses that may possibly have non-smooth or not strictly convex Bayes risks.

Consider a proper loss $l(\pp, z) = \lambda_F(\subgrad F^\ast(\pp), z)$.
If $\subgrad F^\ast(\pp)$ is invertible on $\tldistr$, then $\lambda_F$ can be regarded as a composite proper loss with a link function $\subgrad F^\ast (\pp)$~\citep{williamson.etal2016}. In this case, we can use a model that outputs a value on $\domain$ instead of the class posterior probabilities.

This approach has practical advantages. Given $\vv\in\domain$, a loss is just $\lambda_F(\vv, z)$ and is always guaranteed to be convex as a function of $\vv\in\domain$. This may facilitate optimization. In addition, once the best prediction $\hat\vv$ is obtained, it can be converted into class probabilities by using $\pp \in \partial F(\hat\vv)$. That is, we can completely circumvent calculation of the convex conjugate $F^\ast$, which may not be straightforward in general. Because a subdifferential map $\partial F^\ast$ of a strictly convex function $F^\ast$ is injective (see Appendix~\ref{apdx:strictly-proper}), it follows that $\subgrad F^\ast(\pp)$ is invertible if $F^\ast(\pp)$ is strictly convex, or equivalently, if $l(\pp,z)$ is strictly proper.

Note that $F^\ast$ and $-\underline{L}$ are different in a subtle way, though they are closely related:
the Bayes risk is defined only on $\tldistr$, while $F^\ast$ has a larger domain.
For example, $F^\ast$ might be finite at points in $\aff \tldistr \setminus \tldistr$, where $\aff \tldistr$ represents the affine hull of $\tldistr$.
It also holds that $F^\ast(\pp) = F^\ast (\pp+t\onez)$ for all $p\in\tldistr$ and $t \in \real$, but $\pp + t\onez$ is not in $\tldistr$ if $t\neq 0$.

This might lead us to suspect that minimizing $\lambda_F(\vv, z)$ can result in a solution that does not correspond to posterior probabilities in $\tldistr$. Indeed, even if $F$ satisfies the conditions in Theorem~\ref{thm:dual-rep}, there might be a point $\vv\in\domain$ for which any solution of $\subgrad F^\ast(\pp) = \vv$ does not belong to $\tldistr$. This is because the theorem guarantees the convex conjugate of $F$ to be well-defined in $\tldistr$ but also allows it to exist outside $\tldistr$. However, the following proposition, which is proved in Appendix~\ref{prf:min-in-domain}, guarantees that minimizers of the loss always correspond to some point in $\tldistr$.
\begin{proposition}\label{prop:min-in-domain}
Let $F: \domain \to \real$ be a convex function that satisfies the conditions in Theorem~\ref{thm:dual-rep}.
Then, any minimizer $\vv$ of 
$
    \expect_{z\sim\pp} \left[
        \lambda_F(\vv, z)
    \right]
    = \sum_{z\in\truelabels} p_z \lambda_F(\vv, z)
$, if one exists,
satisfies $\subgrad F^\ast(\qq) = \vv$ for some $\qq \in \tldistr$, where $F^\ast$ is the convex conjugate of $F$.
\end{proposition}

\section{\texorpdfstring{Characterization of $T$-Proper Losses}{Characterization of T-Proper Losses}}
\label{sec:unbounded-T-proper}

In this section, we characterize $T$-proper losses, which may possibly be lower-unbounded.
Our main theorem here is closely related to backward correction in that it involves inversion of a label-corruption process.
However, because our result gives necessary and sufficient conditions for $T$-properness, it also holds for forward-corrected losses and any other $T$-proper losses.

For a closed convex function $F:\domain\subset \onezperp \to \real$ and a reconstruction matrix $R$ for weak labels $\weaklabels$, we define a function $\lambda_{F,R}: \domain\times\weaklabels \to \real$ as
\begin{align}
    \lambda_{F,R}(\vv, y) = - (R\T \vv)_y + F(\vv).
\end{align}
Then we can state the main theorem of this section as follows:

\begin{theorem} \label{thm:T-proper-dual-rep}
Let $T$ be a label transition matrix for weak labels $\weaklabels$, and let $l_\mathrm{W}: \tldistr\times\weaklabels \to \extreal$ be a weak-label loss.
Then, $l_\mathrm{W}$ is $T$-proper if and only if there exist a closed convex function $F:\domain\subset\onezperp \to \real$, a reconstruction matrix $R$ of $T$, and a function $\vec{\Delta}(\qq)$ taking values on the cokernel\footnote{The cokernel of $T$ is the kernel of $T\T$, i.e., a set of vectors $\vv$ in $\real^{\weaklabels}$ such that $T\T\vv =\vec{0}$.} of $T$, which satisfy the following conditions:
\begin{enumerate}
    \item $F(\vv) - \max_{z\in\truelabels} v_z$ is bounded from below.
    \item It holds that $l_\mathrm{W}(\qq, y) = {\lambda}_{F,R}(\subgrad F^\ast(\qq), y) + {\Delta}_y(\qq)$, where $\subgrad F^\ast(\qq)$ is an appropriately chosen subgradient function.
\end{enumerate}
\end{theorem}

See Appendix~\ref{prf:T-dual} for a proof.

Because of the assumption of reconstructibility, we have that $\vec{\Delta}(\qq) \equiv \vec{0}$ if $|\truelabels| = |\weaklabels|$.
On the other hand, if $|\truelabels| < |\weaklabels|$, a label transition matrix $T$ has a cokernel of nonzero dimension, and therefore, $\vec{\Delta}(\qq)$ might take finite values.
However, even if $\vec{\Delta}(\qq)\neq 0$ for some $\qq$, by the definition of $\coker T$, we have that $\inner{T\pp}{\vec{\Delta}(\qq)}=0$ for all $\pp, \qq \in \tldistr$, which leads to the following proposition:
\begin{proposition}
The function $\vec{\Delta}(\qq)$ in Theorem~\ref{thm:T-proper-dual-rep} does not contribute to the expected loss; that is, $\expect_{y\sim T\pp} [\Delta_y(\qq)] = 0$ for all $\pp, \qq\in\tldistr$.
In particular, it holds that $\vec{\Delta}(\qq) \equiv \vec{0}$ if $\card{\truelabels}=\card{\weaklabels}$.
\end{proposition}

Two well-known classes of $T$-proper losses are forward and backward correction losses.
Because Theorem~\ref{thm:T-proper-dual-rep} is applicable to any $T$-proper loss, the loss functions of these classes also conform to it. We demonstrate this in the following two examples.

\begin{example}[Forward correction]
Let $l_\weaklabels: \wldistr\times\weaklabels \to \extreal$ be a proper loss for estimating weak-label posterior probabilities.
Note the difference from a weak-label loss $l_\mathrm{W}: \tldistr\times\weaklabels\to\extreal$ and a proper loss $l:\tldistr\times\truelabels \to \extreal$ for supervised learning.
A weak-label loss $l_\mathrm{W}: \tldistr\times\weaklabels\to \extreal$ is called the forward correction of $l_\weaklabels$ if $l_\mathrm{W}(\qq, y) = l_\weaklabels(T\qq, y)$.
Its $T$-properness is a consequence of the properness of $l_\weaklabels$ and the reconstructibility of $T$.
In Appendix~\ref{apdx:forwardcorrection}, we prove that forward correction losses conform to Theorem~\ref{thm:T-proper-dual-rep}.
\qed
\end{example}

\begin{example}[Backward correction]
Let $l:\tldistr\times\truelabels \to \extreal$ be a proper loss for fully supervised learning.
A backward-corrected loss $l_\mathrm{W}:\tldistr\times\weaklabels \to \extreal$ associated with $l$ is defined as $l_\mathrm{W}(\qq, y) = \sum_{z\in\truelabels} R_{zy} l(\qq, z)$.
By applying Theorem~\ref{thm:dual-rep} to $l(\qq,z)$, we find that $l_\mathrm{W}(\qq, y) = - [R\T \subgrad F^\ast(\qq)]_y + F(\subgrad F^\ast(\qq) ) (R\T \onez)_y$.
It can be shown that $R\T \onez-\oney \in \coker T$ (see Appendix~\ref{apdx:R} for a proof).
Therefore, the backward-corrected loss has the form given in Theorem~\ref{thm:T-proper-dual-rep} with $\vec{\Delta}(\qq)=F(\subgrad F^\ast(\qq) ) (R\T \onez-\oney)$.
For any label transition matrix $T$, we can choose a reconstruction matrix $R$ such that $R\T \onez = \oney$ (see Appendix~C.6 in \citet{vanrooyen.williamson2018} and Appendix~\ref{apdx:R} in this paper); therefore, we can always make $\vec{\Delta}(\qq)$ zero for a backward-corrected loss by using an appropriate $R$.
\qed
\end{example}

\section{\texorpdfstring{Lower-Boundedness of $T$-Proper Losses}{Lower-Boundedness of T-Proper Losses}}

$T$-proper losses as constructed in Theorem~\ref{thm:T-proper-dual-rep} may not be bounded from below. Indeed, there is a gap between the boundedness criteria for proper losses and $T$-proper losses.
In Section~\ref{sec:unbounded}, we see this for an example of the softmax cross-entropy loss. In Section~\ref{sec:suff-cond}, we give a sufficient condition under which a $T$-proper loss is bounded from below.

\subsection{\texorpdfstring{$T$-Proper Loss May Not Be Bounded from Below}{T-Proper Loss May Not Be Bounded from Below}}\label{sec:unbounded}

Consider a $T$-proper weak-label loss $l_\mathrm{W}(\qq, y) = \lambda_{F,R}(\subgrad F^\ast(\qq), y)$ with $\vec{\Delta}(\qq) =\vec{0}$.
To see if $\lambda_{F,R}(\vv, y)$ is bounded from below, we need to compare $F(\vv)$ with $R\T \vv$.
On the other hand, any regular proper loss is bounded from below, because the definition of regularity requires that the loss must not be negative infinity on the compact probability simplex.
This is also reflected in Condition~1 of Theorem~\ref{thm:T-proper-dual-rep}, which suffices to ensure the lower-boundedness of a loss of the form $- v_z +F(\vv)$.
The following lemma implies that the boundedness of $T$-proper losses imposes a stronger restriction on $F(\vv)$ than that of proper losses.

\begin{lemma}\label{lemma:max-Rv-max-v}
Let $R$ be a reconstruction matrix. Then
$\max_{y\in\weaklabels} ( R\T \vv )_y \geq \max_{z\in\truelabels} v_z $ for any vector $\vv\in\onezperp$.
\end{lemma}

See Appendix~\ref{prf:maxRv} for a proof.

\begin{figure}
    \vskip0.2in
    \begin{minipage}{0.45\hsize}
        \centering
        \includegraphics[width=\hsize]{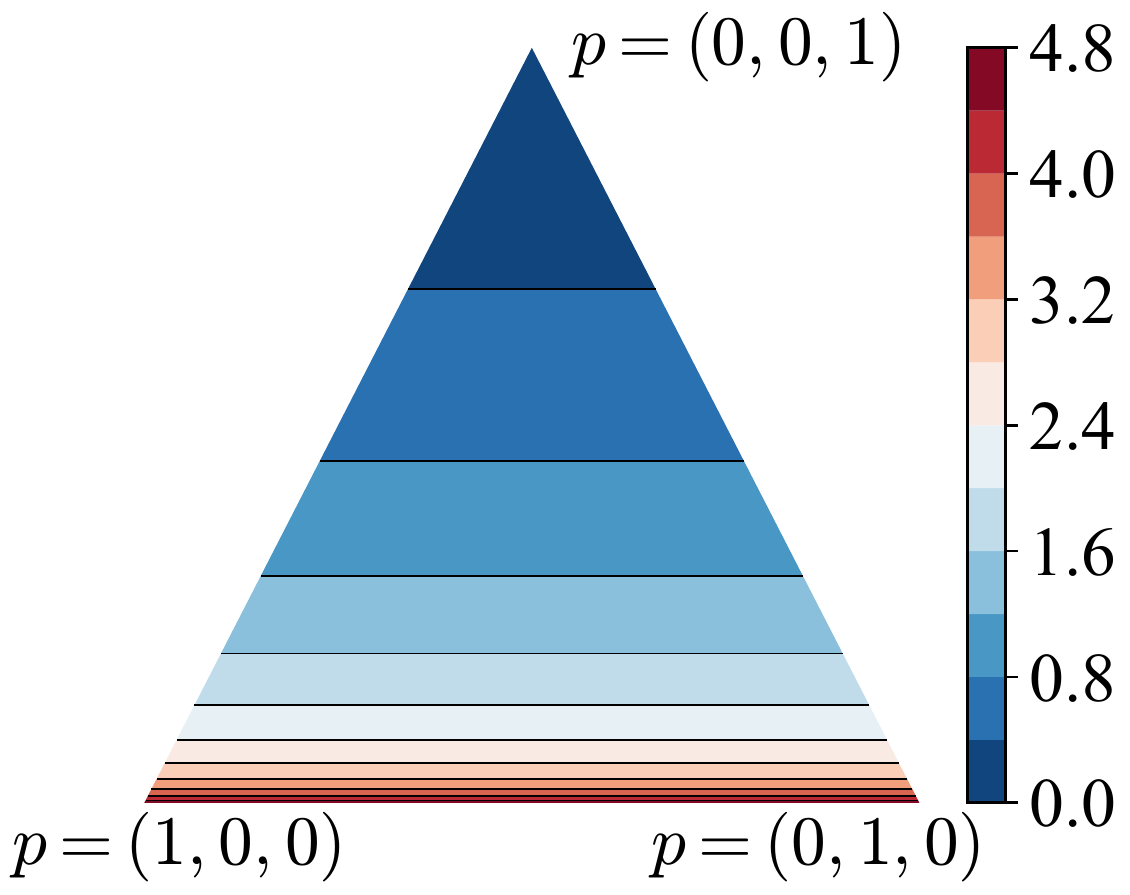}
        \subcaption{$y=(0, 0, 1)$}
    \end{minipage} \hfill
    \begin{minipage}{0.45\hsize}
        \centering
        \includegraphics[width=\hsize]{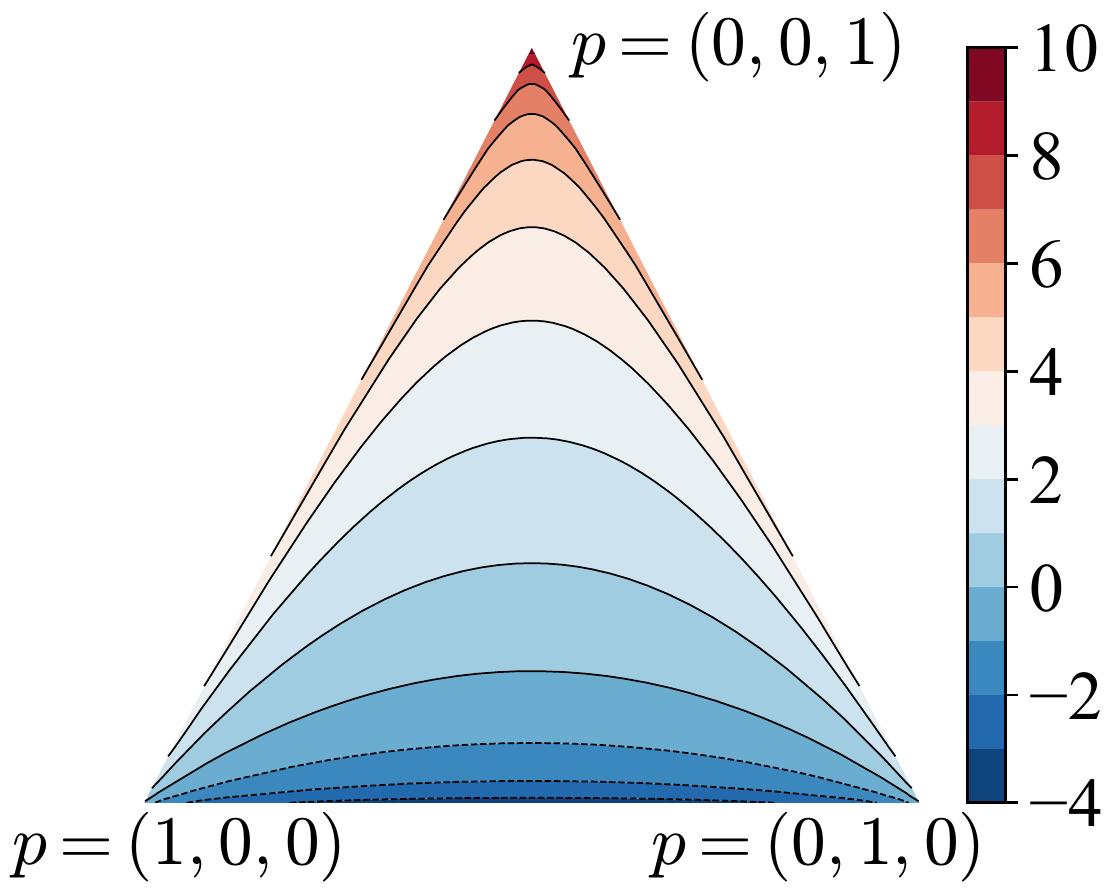}
        \subcaption{$y = (1, 1, 0)$}
    \end{minipage} \\[3mm]
    \begin{minipage}{0.45\hsize}
        \centering
        \includegraphics[width=\hsize]{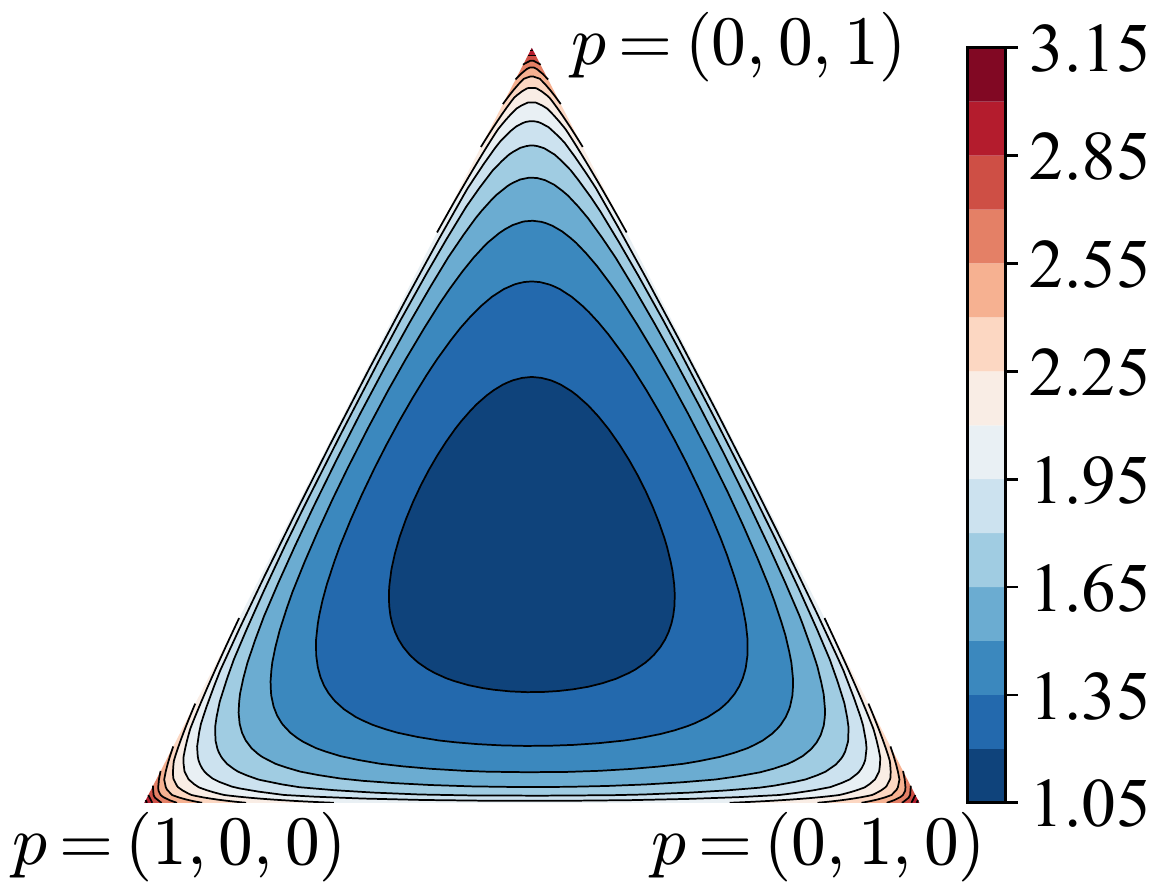}
        \subcaption{$y=(1,1,1)$}
    \end{minipage} \hfill
    \begin{minipage}{0.45\hsize}
        \centering
        \includegraphics[width=\hsize]{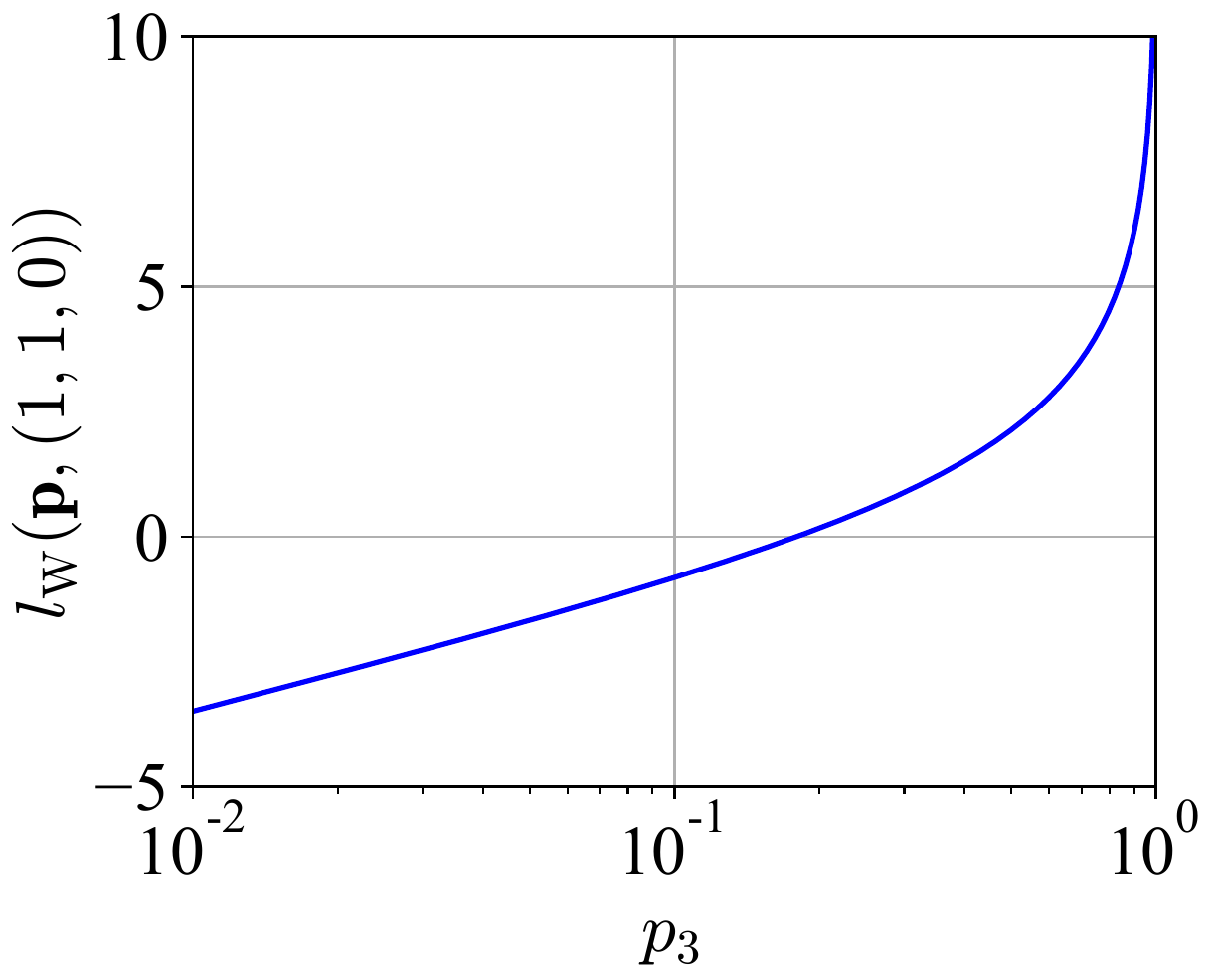}
        \subcaption{$y = (1, 1, 0), \  p_1=p_2=(1-p_3)/2$}
    \end{minipage}
    \caption{Weak-label loss $l_\mathrm{W}(\pp, y)$ with $F(\vv)=\sum_{z\in\truelabels} \log(e^{v_z})$ for partial labels (Example~\ref{ex:unbounded}). (a--c) $l_\mathrm{W}(\pp, y)$ for all $\pp\in\tldistr$. By symmetry, the plots for weak labels $y$ that are not shown here can be obtained by rotating one of these plots. (d) $l_\mathrm{W}(\pp, (1,1,0))$ for $\pp\in\tldistr$ such that $p_1 = p_2 = (1-p_3)/2$.}
    \label{fig:unbounded}
\end{figure}
\begin{example}\label{ex:unbounded}
Consider $F(\vv)=\log (\sum_{z\in\truelabels} e^{v_z})$, which corresponds to the softmax cross-entropy loss and satisfies $F(\vv) > \max_{z\in\truelabels} v_z$ for all $\vv\in\truelabels$.
Also, up to an exponentially small correction, it holds that $F(t\vv) \simeq t\max_{z\in\truelabels} v_z$ for $\vv\neq\vec{0}$ and large positive $t$.
This fact and Lemma~\ref{lemma:max-Rv-max-v} imply that $\lambda_{F,R} (t\vv, y) \simeq t [- ( R\T \vv )_y  + \max_{z\in\truelabels} v_z] \leq 0$ for $y\in\argmax_{y\in\weaklabels} ( R\T \vv )_y$.
If we choose $\vec{v}$ such that this inequality is strict, this diverges to negative infinity as $t\to\infty$.
Therefore, the weak-label loss constructed by applying Theorem~\ref{thm:T-proper-dual-rep} to this $F(\vv)$ with $\vec{\Delta}(\qq)=\vec{0}$ is not bounded from below.
To provide a concrete example, we examine partial labels as described in Example~\ref{ex:partial-label}. Here, we take a reconstruction matrix
\begin{align}
    R = \begin{pmatrix}
    1 & 0 & 0 & \frac{3-2p}{3(1-p)} & \frac{3-2p}{3(1-p)} & - \frac{3-p}{3(1-p)} & \frac{1}{3} \\
    0 & 1 & 0 & \frac{3-2p}{3(1-p)} & - \frac{3-p}{3(1-p)} & \frac{3-2p}{3(1-p)} & \frac{1}{3} \\
    0 & 0 & 1 & - \frac{3-p}{3(1-p)} & \frac{3-2p}{3(1-p)} & \frac{3-2p}{3(1-p)} & \frac{1}{3}
    \end{pmatrix},
\end{align}
and we set $p=0.1$.
Figures~\ref{fig:unbounded}(a--c) show contour plots of $l_\mathrm{W}(\pp, y)$ for all $\pp\in\tldistr$ and $y=(0,0,1),(1,1,0)$, and $(1,1,1)$.
We can just rotate these plots to find the plots for the other weak labels.
Among these, $l_\mathrm{W}(\pp, (1,1,0))$ is not bounded from below.
To make the divergence clearer, Fig.~\ref{fig:unbounded}(d) shows the same function on the line satisfying $p_1=p_2$.
The plot suggests that the loss indeed diverges logarithmically to negative infinity, or equivalently, it diverges linearly in the logit, which is consistent with the above discussion.
\qed
\end{example}

\subsection{Sufficient Condition for Lower-Boundedness}
\label{sec:suff-cond}

Now, we are ready to state a sufficient condition for $T$-proper losses to be bounded from below.
Lemma~\ref{lemma:max-Rv-max-v} implies that if $F(\vv) - \max_{y\in\weaklabels} (R\T \vv)_y$ has a lower bound on $\domain$, then Condition~1 in Theorem~\ref{thm:T-proper-dual-rep} is automatically satisfied. Therefore, we have the following theorem.

\begin{theorem} \label{thm:positive-T-proper}
Let $T$ be a label transition matrix for weak labels $\weaklabels$, and let $F:\domain\subset\onezperp \to \real$ be a closed convex function. If $F(\vv) - \max_{y\in \weaklabels} (R\T \vv)_y$ is bounded from below in $\domain$, then a weak-label loss $l_\mathrm{W}(\qq, y) = \lambda_{F,R}(\subgrad F^\ast(\qq), y)$ is $T$-proper and lower-bounded, where $R$ is a reconstruction matrix of $T$,  $\subgrad F^\ast(\qq)$ is a subgradient function of the convex conjugate $F^\ast(\qq)$ of $F(\vv)$, and the function $\lambda_{F,R}$ is defined as $\lambda_{F,R}(\vv, y) = - (R\T \vv)_y + F(\vv)$.
\end{theorem}

Theorem~\ref{thm:positive-T-proper} gives a sufficient condition for a $T$-proper loss to have a lower bound, but it is not necessary. For example, a $T$-proper loss is not of the above form whenever it has a contribution of $\vec{\Delta}(\qq)$, as in Theorem~\ref{thm:T-proper-dual-rep}, that cannot be absorbed in $\lambda_{F,R}(\qq, y)$. Still, Theorem~\ref{thm:positive-T-proper} gives a large class of lower-bounded $T$-proper losses that are parameterized by a convex function $F$ that is constrained only by a single inequality.

We can also interpret the condition in Theorem~\ref{thm:positive-T-proper} in terms of its dual, or the Bayes risk.
Crudely speaking, the condition can be understood as a constraint to ensure that the Bayes risk is finite at ``class probabilities given a weak label.''
More precisely, the following proposition paraphrases the condition imposed on $F$ in Theorem~\ref{thm:positive-T-proper} into a condition on $F^\ast$.
See Appendix~\ref{prf:equiv} for a proof.
\begin{proposition}\label{prop:equiv}
Let $R$ be a reconstruction matrix for weak labels $\weaklabels$, and let $F:\domain\subset \onezperp \to \real$ be a closed convex function.
Then, $F(\vv) - \max_{y\in \weaklabels} (R\T \vv)_y$ is bounded from below in $\domain$ if and only if $F^\ast(R \vec{e}_y) < \infty$ for all $y\in\weaklabels$, where $\vec{e}_y\in\wldistr$ is a distribution over weak labels that concentrates on a single weak label $y$.
\end{proposition}

The condition $F^\ast(R \vec{e}_y) < \infty$ can be informally paraphrased as $\underline{L}(R\vec{e}_y) > -\infty$, because $F^\ast$ and $-\underline{L}$ are equal in $\tldistr$.
Note, however, that $R\vec{e}_y$ is not necessarily in $\tldistr$ because of the negative components of $R$, and therefore, $F^\ast$ and $-\underline{L}$ may not be equal at $R\vec{e}_y$.

A function $\lambda_{F,R} (\vv, y)$ is convex as a function of $\vv\in\domain$, because it is a sum of the linear function $(R\T \vv)_y$ and a convex function $F(\vv)$~\citep{vanrooyen.williamson2018}.
As with proper losses, therefore, we can obtain the benefits of the convexity of $\lambda_{F,R} (\vv, y)$ by using $\domain$-valued models.

\subsection{Generalized Logit Squeezing}

If we note that $\max_{y\in \weaklabels} (R\T \vv)_y$ is a positively homogeneous function of degree 1, then any convex function $F$ that grows superlinearly satisfies the condition of Theorem~\ref{thm:positive-T-proper}.
This fact leads to the following corollary of Theorem~\ref{thm:positive-T-proper}, which gives a useful way to regularize an unbounded loss.
\begin{corollary}\label{cor:squared-logits}
Let $F:\domain\subset\onezperp \to \real$ be a convex function, let $\alpha$ be a real number that is greater than 1, and let $k$ be a positive number.
We define a convex function $F'$ as
\begin{align}
    F'(\vv) 
    = F(\vv) + \frac{k}{2} \sum_{z\in\truelabels} |v_z|^\alpha
    \label{eq:gLS}
\end{align}
for $\vv \in \domain$.
Then, a weak-label loss $l_\mathrm{W}(\qq, y) = \lambda_{F',R}(\subgrad {F'}^{\ast}(\qq), y)$ is $T$-proper and lower-bounded.
\end{corollary}

\begin{algorithm}[t]
\renewcommand{\algorithmicrequire}{\textbf{Input:}}
\renewcommand{\algorithmicensure}{\textbf{Output:}}
\caption{Training of the linear model with the backward-corrected cross entropy and generalized logit squeezing\label{algo}.}
\begin{algorithmic}
\REQUIRE training data $D=\{(\vec{x}_i, \vec{y}_i) \}$, reconstruction matrix $R$, coefficient $k$, exponent $\alpha$, batch size $N$, SGD-like algorithm $\mathcal{A}$.
\ENSURE weight matrix $W$
\STATE{Initialize weights $W$}
\REPEAT
    \STATE Sample minibatch $(X, \vec{y})$ from $D$
    \STATE $V \leftarrow X W\T$
    \STATE $l_\mathrm{ce} \leftarrow \frac{1}{N} \sum_{i=1}^N \left[-(VR)_{i y_i} + \log \sum_z \exp(v_{iz}) \right]$
    \STATE $l_\mathrm{gLS} \leftarrow \frac{1}{N}\sum_{i=1}^N \sum_z \frac{k}{2} |v_{iz}|^\alpha$
    \STATE $l \leftarrow l_\mathrm{ce} + l_\mathrm{gLS}$
    \STATE Update $W$ by using an algorithm $\mathcal{A}$
\UNTIL{a stopping criterion is met}
\end{algorithmic}
\end{algorithm}
To facilitate the use of this corollary, we present pseudocode for training the linear model with the backward-corrected cross entropy loss in Algorithm~\ref{algo}.

The term $\sum_{z\in\truelabels} |v_z|^\alpha$ is convex if and only if $\alpha\geq1$.
The conclusion of the corollary for $\alpha = 1$ depends on the precise form of $F$ and the value of $k$.

Corollary~\ref{cor:squared-logits} indicates that if a $T$-proper loss associated with $F$ is not bounded from below, then we can replace $F$ with $F+\frac{k}{2}\sum_{z\in\truelabels}| v_z |^\alpha$ to make the loss bounded while keeping its $T$-properness.
We refer to the proposed regularization scheme of Eq.~\eqref{eq:gLS} as generalized logit squeezing (gLS).
The special case with $\alpha=2$ has the same form as the regularization schemes called feature contraction~\cite{li.maki2018} and logit squeezing~\cite{kannan.etal2018}.
Those previous studies focused on the performance of supervised learning~\cite{li.maki2018} or adversarial robustness~\cite{kannan.etal2018}, and they were mostly empirical.
On the other hand, gLS has a solid theoretical foundation that guarantees its asymptotic success in weakly supervised learning.

Although gLS might appear similar to $L^p$ regularization, they are different concepts.
gLS penalizes a model's large output values, whereas normal $L^p$ regularization pulls training trajectories toward smaller norms of the weights.
They both restrict the model space but in different ways, and their actual effects on learning might be very different.
On the other hand, if $\vv\in\domain$ is a linear function of the weights, then gLS is closely related to $L^p$ regularization, because the gLS term is a positively homogeneous function of degree $\alpha$ on the weights.
In this particular case, the two regularization schemes could be expected to work in a similar way.

\section{Experiment}

\begin{table}
    \centering
    \caption{Comparison of three losses.
    BC, GA, and gLS respectively stand for backward correction, gradient ascent, and generalized logit squeezing.
    \label{tab:loss-comparison}}
    \vskip 0.15in
    \begin{small}
    \begin{tabular}{ccc}
    \toprule
     & Proper & Bounded \\
     \midrule
     BC & \cmark & \xmark \\
     BC + GA & \xmark & \cmark \\
     BC + gLS & \cmark & \cmark \\
    \bottomrule
    \end{tabular}
    \end{small}
    \vskip -0.1in
\end{table}
\begin{table*}
    \centering
    \caption{
    Mean and sample standard deviation of the test accuracy.
    The best accuracy for each dataset and model is shown in boldface.
    BC: backward correction; GA: gradient ascent; gLS: generalized logit squeezing, with the exponent $\alpha$ fixed to $2$.}
    \label{tab:results}
    \vskip 0.15in
    \begin{small}
    \begin{tabular}{lccccc}
        \toprule
        &
            Weight decay &
            MNIST, linear &
            MNIST, MLP &
            CIFAR-10, ResNet-20 &
            CIFAR-10, WRN-28-2 \\
        \midrule
        BC & fixed & 
            $81.52 \pm 1.44 \,\%$ &
            $83.09 \pm 0.67 \,\%$ &
            $28.86 \pm 2.06 \,\%$ &
            $29.57 \pm 1.58 \,\%$ \\
        BC & tuned & 
            $83.56 \pm 0.87 \,\%$ &
            $83.30 \pm 1.01 \,\%$ &
            $29.56 \pm 1.49 \,\%$ &
            $30.02 \pm 1.49 \,\%$ \\
        BC + GA & fixed &
            $78.57 \pm 1.82 \,\%$ &
            $87.88 \pm 1.11 \,\%$ &
            $34.39 \pm 2.96 \,\%$ &
            $36.87 \pm 2.26 \,\%$ \\
        BC + GA & tuned &
            $80.63 \pm 1.01 \,\%$ &
            $\mathbf{89.15} \pm 0.75 \,\%$ &
            $35.36 \pm 1.80 \,\%$ &
            $36.90 \pm 2.52\,\%$ \\
        BC + gLS & fixed &
            $\mathbf{83.77} \pm 0.55 \,\%$ &
            $88.63 \pm 0.38 \,\%$ &
            $\mathbf{49.71} \pm 3.04 \,\%$ &
            $\mathbf{49.98} \pm 2.59 \,\%$ \\
        \bottomrule
    \end{tabular}
    \end{small}
    \vskip -0.1in
\end{table*}
In this section, we experimentally compare three different losses, all of which derive from the cross-entropy loss, to demonstrate the effectiveness of lower-bounded proper losses\footnote{The code is publicly available at \url{https://github.com/yoshum/lower-bounded-proper-losses}.}.
Table~\ref{tab:loss-comparison} summarizes these losses.
We take the backward correction (BC) of the softmax cross entropy as a baseline loss:
\begin{align}
    \lambda_{F,R}(\vv, y)
    = - (R\T \vv)_y 
        + \log \sum_{z\in\truelabels} e^{v_z},
\end{align}
which is proper but lower-unbounded.
Here, $\vec{v}$ is so-called logits, which can be converted into class posterior probabilities with the softmax function.
In our experiments, this loss is made bounded from below in two different ways.
The first way is to apply gLS to the backward-corrected cross entropy and use
\begin{align}
    \lambda_{F,R}(\vv, y)
    = - (R\T \vv)_y 
        + \log \sum_{z\in\truelabels} e^{v_z}
        + \frac{k}{2} \sum_{z\in\truelabels} |v_z|^\alpha.
    \label{eq:squared-logits}
\end{align}
For brevity, we refer to this loss as BC + gLS.
It is proper and lower-bounded if $\alpha > 1$, while it becomes improper and lower-unbounded if $\alpha < 1$.
The properties for the boundary case of $\alpha=1$ depend on the value of $k$.
The other way to make the loss bounded from below is to use gradient ascent (GA)~\cite{kiryo.etal2017,ishida.etal2019,lu.etal2020a}, which updates a model in the ascending direction of the loss surface when the empirical class-conditional risk becomes negative.
GA makes the training objective bounded from below but improper.

\subsection{Setup}

As a specific example of weak labels, we experimented with complementary labels~\cite{ishida.etal2017,yu.etal2018,ishida.etal2019}.
Let $c$ be any category label.
Then, a complementary label $\overline{c}$ put on an instance indicates that it belongs to a category other than $c$.
For $K$-class classification, (unbiased) complementary labels are characterized by the following transition matrix:
\begin{align} \label{eq:Tcomp}
    T = \frac{1}{K-1} (1_K - I_K),
\end{align}
where $1_K$ is the $K\times K$ matrix with all elements 1, and $I_K$ is the $K\times K$ identity matrix.
This can be seen as an extreme case of noisy labels, where a label is corrupted with probability 1.

We evaluated the effectiveness of the losses on the MNIST~\cite{lecun.etal1998} and CIFAR-10~\cite{krizhevsky2009} datasets.
To each instance in these datasets, we randomly assigned a complementary label with conditional probabilities given the ground-truth category, which is given by Eq.~\eqref{eq:Tcomp}.
For each dataset, we trained two models: a linear model and a feed-forward network with one hidden layer (multilayer perceptron; MLP) were used with MNIST, and ResNet-20~\cite{he.etal2016} and Wide-ResNet~(WRN)~28-2~\cite{zagoruyko.komodakis2016} were used with CIFAR-10.
We used stochastic gradient descent with momentum to optimize the models.
The momentum was fixed to $0.9$, while the initial learning rates were chosen as those giving the best validation accuracy.
The default value of the weight decay was $10^{-4}$, but we also tuned it with BC and BC~+~GA to compare its effect with that of gLS.
More details on the experimental procedure and the hyperparameters are given in Appendix~\ref{apdx:experiment}.

\subsection{Results}

In Table~\ref{tab:results}, we list the mean and the sample standard deviation of the test accuracy for 16 trials with the chosen hyperparameters.
Here, we tuned the coefficient $k$ and fixed the exponent $\alpha=2$ for BC~+~gLS.
If the weight decay was fixed to the default value for all the models, BC~+~gLS achieved the best test accuracy by a clear margin.
It was particularly effective on the CIFAR-10 benchmark, which used more complex models.
This is reasonable, because complex models are easier to fit to an unbounded training loss and are affected more severely than simple models; therefore, they are more sensitive to how the lower-unboundedness is prevented by regularization.

On the other hand, the effects of tuning the weight decay were not consistent among the different models.
In the experiments with BC \textit{without} GA, we found that tuning the weight decay brought a gain of approximately 2\% to the linear model, which enabled BC to achieve performance comparable to that of the proposed method (BC~+~gLS), but the gains were insignificant for the other models.
This is consistent with the observation that in the linear model, the two regularization schemes have similar functional forms, as explained in the previous section.
By contrast, a larger weight decay seemed to penalize complex models too much and cause underfitting before it prevented loss divergence.
The experiments with BC~+~GA showed a similar trend, except for the MLP model, which exhibited a gain of about 2\% from tuning the weight decay.
Overall, the weight decay could narrow or close the gap between the baselines and the proposed method for simpler models, but it did not have a significant effect for deeper models, which more severely suffer from overfitting due to a lower-unbounded loss.

\begin{figure}
\vskip 0.1in
    \centering
    \includegraphics[width=\hsize]{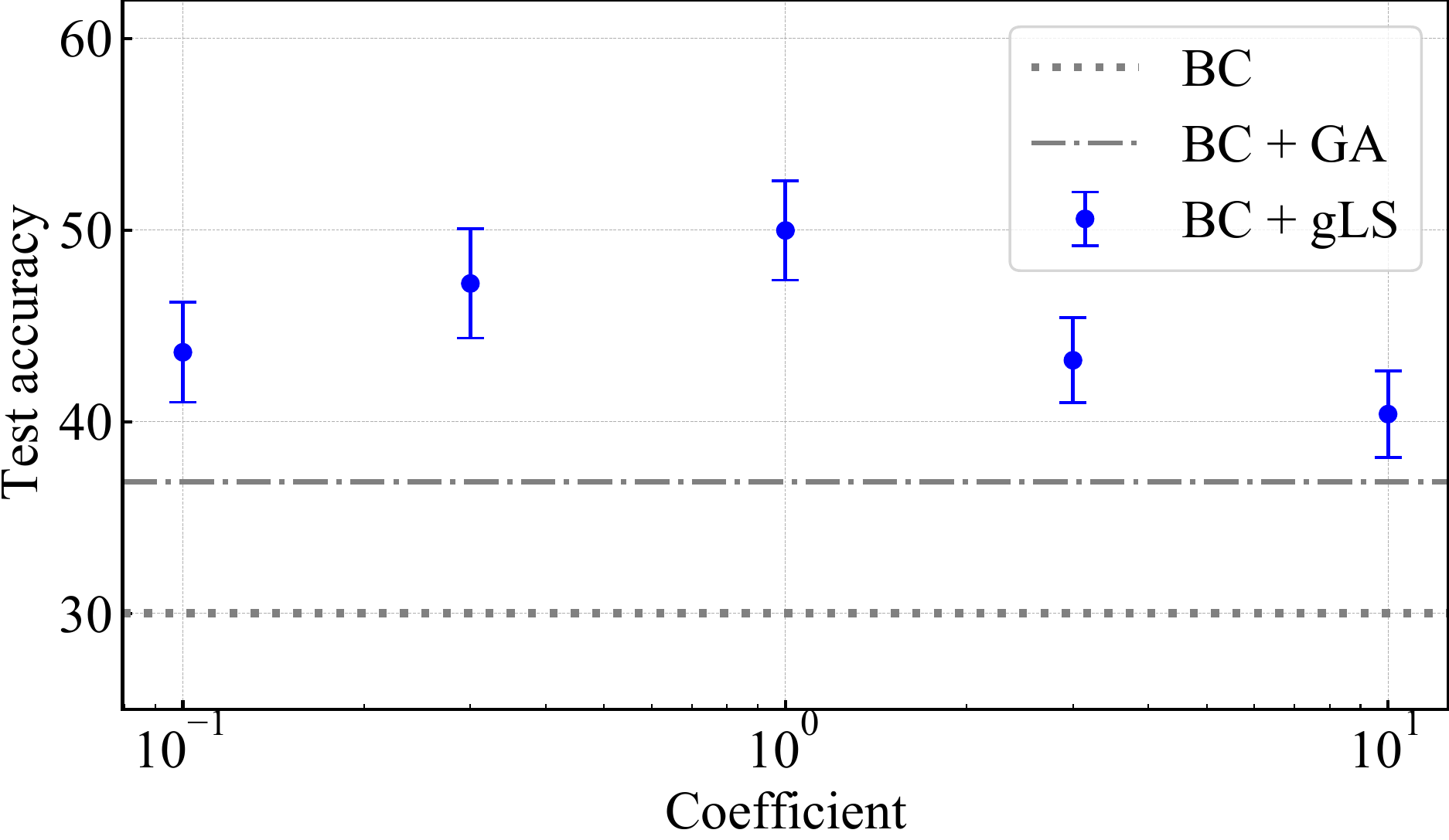}
    \caption{Sensitivity of the test accuracy to the coefficient of the squared logit term in Eq.~\eqref{eq:squared-logits}.
    The bars represent the sample standard deviations.
    The horizontal lines indicate the test accuracies of BC (dotted) and BC~+~GA (dashed-dotted) with the weight decay coefficients that achieved the best validation accuracy.}
    \label{fig:coeff}
    \vskip -0.1in
\end{figure}
We also examined how sensitive the accuracy is to the coefficient $k$ of the gLS term in Eq.~\eqref{eq:squared-logits}.
Figure~\ref{fig:coeff} shows the test accuracies for different $k$ values on CIFAR-10 trained with WRN-28-2, and it indicates two findings.
First, the test accuracy depended significantly on $k$, and it is thus important to choose an appropriate value of $k$ to obtain the best results.
In a sense, this is an obvious conclusion: both of the limits, $k\to0$ and $k\to \infty$, are undesirable, because the former would converge to BC, while the latter would lead to a model that outputs zero for any input; therefore, there should be an optimal value of $k$.
Second, however, the figure indicates that the results were not too sensitive to $k$ and that BC + gLS yielded better results over two orders of magnitude of $k$ as compared to the other methods.

\begin{figure}
\vskip 0.1in
    \centering
    \includegraphics[width=\hsize]{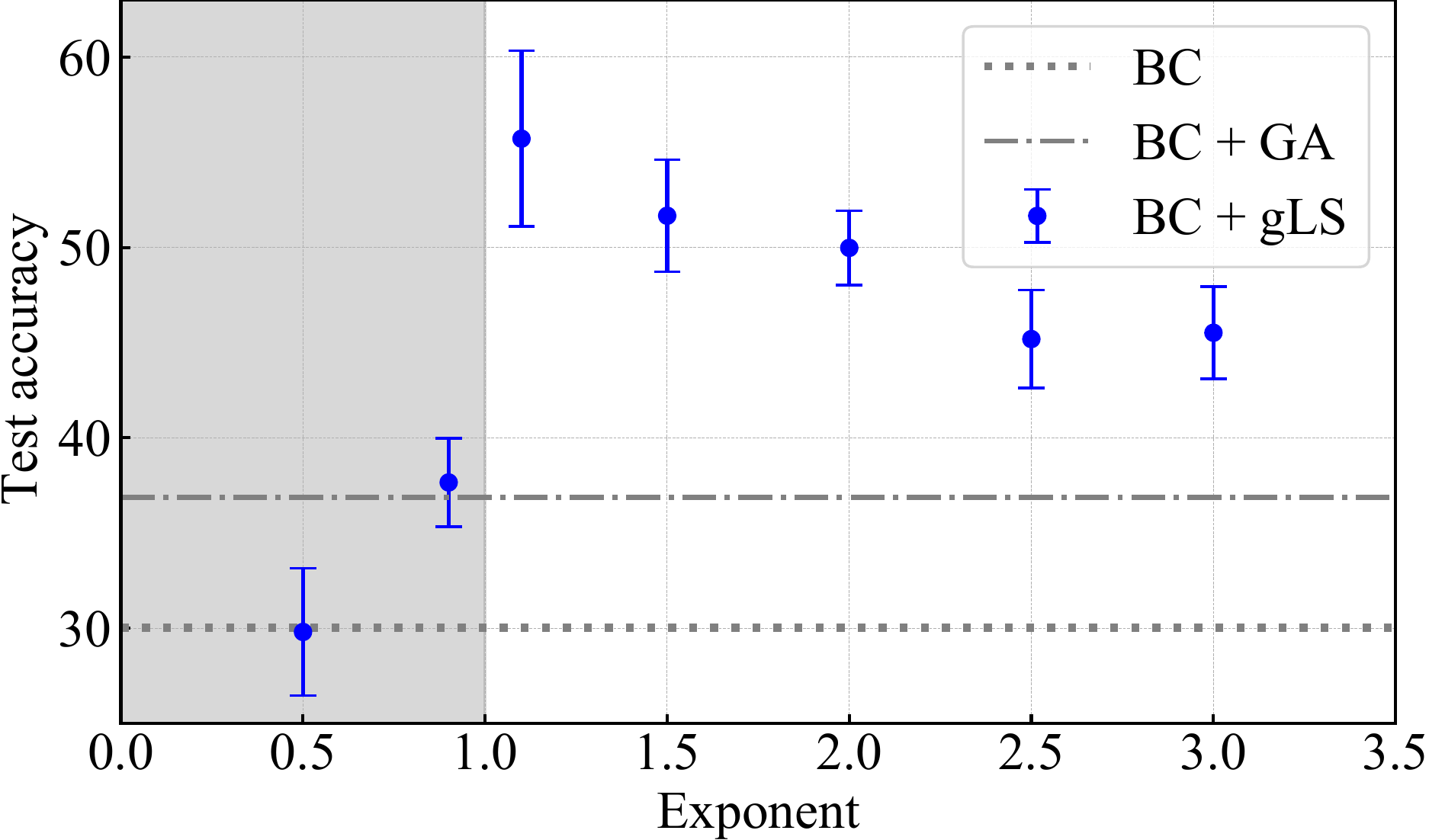}
    \caption{Performance of gLS for different exponents.
    The bars represent the sample standard deviations.
    The gray region ($\alpha < 1$) represents the regime in which gLS yields an improper, lower-unbounded loss.
    The horizontal lines indicate the test accuracies of BC (dotted) and BC~+~GA (dashed-dotted) with the weight decay coefficients that achieved the best validation accuracy.}
    \label{fig:exponent}
    \vskip -0.1in
\end{figure}
As Corollary~\ref{cor:squared-logits} indicates, gLS yields lower-bounded $T$-proper losses as long as $\alpha > 1$.
In Figure~\ref{fig:exponent}, we show the test accuracies on CIFAR-10 trained with WRN-28-2 by using various exponents.
The results demonstrate that gLS gave superior results as compared to the baseline methods, even with $\alpha\neq2$.
Interestingly, the test accuracies improved as $\alpha \to 1$.
As $\alpha$ became less than 1, however, the test accuracies immediately dropped to values similar to the baselines.
This observation not only validates the effectiveness of our proposed approach but also underlines the importance of using $T$-proper and lower-bounded losses, which is the central premise that motivated our theoretical analysis.

\section{Conclusion}

In this paper, we have discussed proper losses for weakly supervised classification. We first derived the dual representation of proper losses for supervised learning. Instead of the Bayes risk, which plays a central role in the Savage representation, the derived theorem represents a loss with a function related to the convex conjugate of the Bayes risk. We then used this theorem to characterize $T$-proper losses and derived a sufficient condition for them to be bounded from below.
These theoretical findings led to a novel regularization scheme called generalized logit squeezing (gLS), which prevents any proper weak-label loss from diverging to negative infinity, while keeping the properness of the original loss.
We also experimentally demonstrated the effectiveness of our proposed approach.
Remarkably, gLS yielded superior results as compared to the baseline methods regardless of the precise values of the hyperparameters that are specific to gLS, as long as those parameters were in the regime in which gLS gives $T$-proper and lower-bounded losses.

\section*{Acknowledgement}

SMY gratefully acknowledges fruitful discussions with Akira Tanimoto and Makoto Terao.
TT was partially supported by JSPS KAKENHI Grant Numbers 20K03753 and 19H04071.

{
    \bibliographystyle{icml2021}
    \bibliography{library}
}

\newpage
\onecolumn
\appendix
\section*{Appendix}

\section{Relation of Reconstructibility to Other Assumptions in the Literature}\label{apdx:left-invertibility}

In this appendix, we compare Assumption (b), i.e., the left-invertibility of the label transition matrix $T$, with other learnability/invertibility assumptions in the literature of weakly supervised learning.

\subsection{Identifiability from \citet{zhang.etal2019d}}

Consider a probability distribution $P(X\vert \theta)$ that is parameterized by a set of parameters $\theta \in \Theta$.
This parametric family of distributions satisfies the identifiability condition if
\begin{align}
    \forall \theta_1, \theta_2 \in \Theta, 
    \theta_1 \neq \theta_2
    \Longrightarrow
    P(X \vert \theta_1) \neq P(X \vert \theta_2).
\end{align}
In other words, if $P(X\vert \theta)$ is perfectly known, then the parameter $\theta$ can be uniquely identified.

\citet{zhang.etal2019d} proved the consistency of their algorithm under several assumptions, among which the most fundamental is identifiability of the posterior probability distributions of weak and true labels.
More precisely, they assumed that for any input pattern $x\in\mathcal{X}$, the posterior probability distribution of true labels, $P(Z\vert x)$, belongs to a parametric family of identifiable probability distributions.
Let $P(Z\vert \theta)$ be a distribution in that family and $\Theta$ be a set of parameters.
In addition, they also assumed that the posterior probabilities of weak labels are also identifiable; that is,
\begin{align}
    \forall \theta_1, \theta_2 \in \Theta,
    \theta_1 \neq \theta_2
    \Longrightarrow
    P(Y \vert \theta_1) \neq P(Y \vert \theta_2),
    \label{eq:weak-label-identifiabiliety}
\end{align}
where $P(Y=y \vert \theta) \equiv \sum_{z\in\truelabels} T_{yz} P(Z=z \vert \theta)$.
Note that $\theta$ is different from model parameters such as weights in neural networks.
Here, a model is a function $f_w: \mathcal{X} \to \Theta$ that is parameterized by a set of network weights $w$.
For a given input $x$, it predicts a parameter $\theta = f_w(x)$ and in turn the posterior probabilities $P(Z\vert \theta)$.

The identifiability of true label distributions is automatically satisfied by careful implementation.
For example, if the categorical posterior probabilities are expressed by using the softmax function, the choice $\Theta=\onezperp$ guarantees identifiability.
Therefore, we use this assumption in the discussion below.

\citet{zhang.etal2019d} claimed that they successfully avoided relying on the existence of a left-inverse of $T$ by resorting to the identifiability assumptions.
However, without any prior knowledge on the true posterior probability distributions, identifiability implies the left-invertibility of $T$.
Specifically, we can prove the following proposition.

\begin{proposition}
Let $T$ be a label transition matrix.
Assume that for any $x\in\mathcal{X}$, a posterior probability distribution $P(Z\vert x)$ of true labels belongs to a parametric family $\{ P(Z\vert \theta) \mid \theta \in \Theta \}$ of identifiable distributions.
Then, the left-invertibility of $T$ implies the identifiability of $P(Y\vert x)$, the posterior probability distribution of weak labels.
Moreover, the converse also holds if $\{ P(Z\vert \theta) \mid \theta \in \Theta \} = \tldistr$.
\end{proposition}
\begin{proof}
Suppose that $T$ is left-invertible.
Then, it holds that
\begin{align}
    P(Z=z \vert \theta) = \sum_{y\in\weaklabels} R_{zy} P(Y=y\vert \theta),
\end{align}
where $R$ is a left-inverse of $T$.
This implies that if $P(Y\vert\theta_1) = P(Y\vert\theta_2)$, then $P(Z\vert\theta_1) = P(Z\vert\theta_2)$, from which it follows that $\theta_1 = \theta_2$ because of the identifiability of $P(Z\vert\theta)$.
Therefore, Eq.~\eqref{eq:weak-label-identifiabiliety} holds.

Conversely, suppose that $P(Y\vert\theta)$ is identifiable and also that $\{ P(Z\vert \theta) \mid \theta \in \Theta \} = \tldistr$.
Let $\theta_1$ and $\theta_2$ be parameters in $\Theta$ such that $\theta_1 \neq \theta_2$, 
let $\vec{\Delta}_Z (\theta_1, \theta_2)$ be a vector in $\real^\truelabels$ with components $[\vec\Delta_Z (\theta_1, \theta_2)]_z= P(Z=z \vert \theta_1) - P(Z=z\vert \theta_2)$, and
let $\vec{\Delta}_Y (\theta_1, \theta_2)$ be a vector in $\real^\weaklabels$ with components $[\vec\Delta_Y (\theta_1, \theta_2)]_y= P(Y=y \vert \theta_1) - P(Y=y\vert \theta_2)$.
By the assumption that $\{ P(Z\vert \theta) \mid \theta \in \Theta \} = \tldistr$, we have that
\begin{align}
    \{ t \vec{\Delta}_Z(\theta_1, \theta_2) \mid t \in \real, \theta_1\in \Theta, \theta_2 \in \Theta \} = \onezperp.
\end{align}
On the other hand, by the assumption that $P(Y\vert\theta)$ is identifiable, for any $\theta_1, \theta_2 \in \Theta$ such that $\theta_1 \neq \theta_2$, it holds that
\begin{align}
    \vec{\Delta}_Y (\theta_1, \theta_2) = T \vec{\Delta}_Z (\theta_1, \theta_2) \neq \vec{0} \quad \text{and} \quad
    \vec\Delta_Z (\theta_1, \theta_2) \neq \vec{0}.
\end{align}
These equations imply that $\vec\Delta_Z (\theta_1, \theta_2)$ is a nonzero vector that does not belong to the kernel of $T$.
They also imply that any vector in the kernel of $T$ is perpendicular to $\onezperp$.
Moreover, $\onez$ is not in the kernel of $T$: if that were the case, the uniform distribution of true labels would be mapped to the zero vector, which does not correspond to any weak-label distribution.
Therefore, the kernel of $T$ is $\{\vec0\}$, which means that $T$ is left-invertible.
\end{proof}

This proposition suggests that the identifiability assumption is equivalent to the left-invertibility of $T$ in cases with $\{ P(Z\vert \theta) \mid \theta \in \Theta \} = \tldistr$.
Indeed, this is what usually happens in practice:
we do not know \textit{a priori} in which subset of $\tldistr$ the true posterior probabilities reside, and thus, it is customary to take $\{ P(Z\vert \theta) \mid \theta \in \Theta \}$ to be $\tldistr$ itself.
We can see from the proof above that if $T$ is not left-invertible, then $\vec{\Delta}_Z(\theta_1, \theta_2)$ must lie outside the kernel of $T$ for any $\theta_1, \theta_2 \in \Theta$ in order for $P(Y\vert\theta)$ to be identifiable.
This constraint implies that $\{ P(Z\vert \theta) \mid \theta \in \Theta \}$ has strictly lower dimensions than $\tldistr$ does, which essentially means that we can exclude some of the labels in $\truelabels$ at the modeling step.

\subsection{Non-Ambiguity Condition in Partial-Label Learning}

In theoretical analyses of partial-label learning, the so-called non-ambiguity condition has been used~\citep{cour.etal2011,cabannes.etal2020}.
In this section, we discuss the relation between the non-ambiguity and the left-invertibility of $T$.

A partial label $y\in \weaklabels$ is a candidate set of labels, only one of which is correct.
Obviously, $\weaklabels \subset 2^{\truelabels} \setminus \{ \emptyset \}$, where $2^{\truelabels}$ is the power set of $\truelabels$.
The empty set $\emptyset$ is not in $\weaklabels$ because a partial label always contains a correct label.

\begin{definition}[Non-ambiguity condition]
Let $P(z'\in Y\setminus \{ z \} \mid Z=z)$ be the probability that a partial label contains an incorrect label $z'$, given a true label $z$.
Then, the ambiguity degree $\epsilon$ is defined as follows\footnote{In the original paper~\citep{cour.etal2011}, the ambiguity degree was defined with the probability conditioned on an input pattern $x$ as well. We omit that conditioning for brevity because we assume that the distribution of weak labels does not depend on $x$.}:
\begin{align}
    \epsilon
    \equiv
    \sup_{z,z' \in \truelabels, P(Z=z)>0}
        P(z'\in Y\setminus \{ z \} \mid Z=z).
\end{align}
Partial labels are said to satisfy the non-ambiguity condition if $\epsilon < 1$.
\end{definition}

The ambiguity degree is the maximum probability of co-occurrence of an incorrect label $z'$ with a correct label $z$.
To gain some intuition into the ambiguity degree and the non-ambiguity condition, let us consider two extreme cases: $\epsilon=0$ and $\epsilon=1$.
The equality $\epsilon=0$ implies that a weak label $y$ is always a singleton $\{z\}$ if the correct label is $z$.
That is, every instance is given only the correct label, and therefore, this is equivalent to supervised learning.
On the other hand, when $\epsilon=1$ and the non-ambiguity condition is not satisfied, there is a pair of labels $z$ and $z'$ in $\truelabels$ such that if a true label of an instance is $z$, an incorrect label $z'$ is always given to that instance as well.

There is a simple example that does not satisfy the non-ambiguity condition but has a left-invertible label transition matrix $T$.
Consider a binary classification problem $\truelabels=\{1, 2\}$ with a partial label set $\weaklabels=\{ \{1\}, \{1,2\}\} $.
If we identify $1$ with the positive label and $2$ with the negative label, this problem is often referred to as positive-unlabeled (PU) learning or learning with totally asymmetric label noise.
The label transition matrix $T$ has the following form:
\begin{align}
    T = \begin{pmatrix}
    r & 0 \\
    1-r & 1
    \end{pmatrix},
\end{align}
where $r$ ($0<r<1$) is the proportion of positively labeled instances in truly positive instances.
This $T$ is left-invertible and yet breaks the non-ambiguity condition, because examples with the correct label $2$ always have a partial label $\{1, 2\}$ including the incorrect label $1$.

We can further show the following proposition.

\begin{proposition}
Suppose that partial labels satisfy the non-ambiguity condition.
If $\|\truelabels \| =2 $ or $3$, then the label transition matrix $T$ is left-invertible.
On the other hand, if $\|\truelabels \| > 3$, then the label transition matrix is not necessarily left-invertible.
\end{proposition}
\begin{proof}
We prove the case with $\|\truelabels\| \leq 3$ by proving its contraposition.
Suppose that $T$ is not left-invertible.
Then, the column vectors $\vec{t}_z$ of $T$ ($z\in\truelabels$) are not linearly independent; that is, there exists $\{ a_z \}_{z\in\truelabels}$ such that
\begin{align}
    \sum_{z\in\truelabels} a_z \vec{t}_z = \vec0
\end{align}
and at least one of $a_z$ is nonzero.
Because $\vec{t}_z\in\wldistr$, it follows that $\sum_{z\in\truelabels} a_z = 0$.
Therefore, without loss of generality, we can assume that one of the following two equations holds:
\begin{align}
    \vec{t}_{z_1} &= \vec{t}_{z_2}, \\
    \vec{t}_{z_3} &= a_{1} \vec{t}_{z_1} + a_2 \vec{t}_{z_2} \quad
    (a_1 > 0, a_2>0).
\end{align}
By noting that $(\vec{t}_z)_y = P(Y=y \vert Z=z)$ and that $z\in y$ if $P(Y=y\vert Z=z) > 0$,
we can see that the former implies
\begin{align}
    P(Y=\{ z_1, z_2 \} \vert Z=z_1) = P(Y=\{ z_1, z_2 \} \vert Z=z_2) = 1,
\end{align}
while the latter implies
\begin{align}
    P(z_3 \in Y \setminus \{ z_1 \} \vert Z=z_1) =P(z_3 \in Y \setminus \{ z_2 \} \vert Z=z_2) = 1.
\end{align}
In either case, we have $\epsilon = 1$, and therefore, the non-ambiguity condition is broken.

If $\|\truelabels\| = 4$, we can find an example $T$ that is left-invertible but non-ambiguous.
One such example is
\begin{align}
    T = \begin{pmatrix}
        0.5 & 0 & 0.5 & 0 \\
        0.5 & 0 & 0 & 0.5 \\
        0 & 0.5 & 0.5 & 0 \\
        0 & 0.5 & 0 & 0.5
    \end{pmatrix}.
    \label{eq:non-ambiguous-non-invertible}
\end{align}
If $\truelabels=\{1,2,3,4\}$ and $T$'s columns from left to right correspond to 1 to 4,
then the rows from top to bottom represent partial labels $(1,3)$, $(1,4)$, $(2,3)$, and $(2,4)$.
We can see that this is not left-invertible by noting that a nonzero vector $(1,1,-1,-1)\T$ is in the kernel of $T$.
On the other hand, the ambiguity degree $\epsilon$ is 0.5, and therefore, the non-ambiguity condition is satisfied.

We can also construct an example for an arbitrary $\|\truelabels\| > 4$ by using Eq.~\eqref{eq:non-ambiguous-non-invertible}.
For instance, the following block diagonal form of $T$ is non-invertible and non-ambiguous:
\begin{align}
    T = \begin{pmatrix}
    T_4 & 0 \\
    0 & I
    \end{pmatrix},
\end{align}
where $T_4$ is the label transition matrix for the first four labels given by Eq.~\eqref{eq:non-ambiguous-non-invertible}, and $I$ is an identity matrix with an appropriate size.
\end{proof}

\subsection{Weak Noise Condition in Learning from Noisy Labels}

Theoretical analyses of noisy-label learning have assumed that the noise rate is smaller than some threshold, which often coincides with the point at which $T$ is not left-invertible~\citep{angluin.laird1988,natarajan.etal2013}.
For example, \citep{natarajan.etal2013} considered label noise that has the following transition matrix:
\begin{align}
    T = \begin{pmatrix}
    1 - r_+ & r_- \\
    r_+ & 1 - r_-
    \end{pmatrix}.
\end{align}
They assumed that $r_+ + r_- < 1$, and $T$ ceases to be left-invertible at the boundary $r_+ + r_- =1$.
However, $T$ recovers left-invertibility for a noise rate above the threshold (e.g., complementary label learning, which can be seen as the extreme case in which labels are flipped with probability 1).
There, our framework is still applicable.

\section{Multiple Weak-Label Datasets}\label{apdx:multisource}

The arguments in the main text deal with scenarios with only one weak-label set $\weaklabels$ and an associated transition matrix $T$.
In this appendix, we show that without making formal changes, we can extend the formulation to scenarios with multiple samples having different noise characteristics.

Let $N$ be the number of training sets. They all have the same true-label set $\truelabels = \{ z_1, z_2, \dots , z_C \}$ and base distribution $p(x, z)$, but each has its own weak-label set $\weaklabels^{(d)} = \{ y_{1}^{(d)}, y_{2}^{(d)}, \dots , y_{C_d}^{(d)} \}$ and label transition matrix $T^{(d)}$.
We show that this problem can be mapped to a problem with a single weak-label set $\weaklabels=\cup_{d=1}^N \weaklabels^{(d)}$ and a label transition matrix.
A partial risk on the $d$th set is defined as $R_d[q(z\vert x)]\equiv\expect_{(x,y)\sim p^{(d)} (x,y)} [l_\mathrm{W}(q(z\vert x), y)]$, where $p^{(d)}(x,y) \equiv \sum_{z\in\truelabels} T^{(d)}_{yz} p(x,z)$. The total risk is defined as a convex combination of the partial risks:
\begin{align}
    R[q(z\vert x) ]
    &\equiv \sum_{d=1}^N \alpha_d R_d[q(z \vert x)] \\
    &\equiv \expect_{(x,y)\sim \sum_{z\in\truelabels}T_{yz}p(x, z)}
        [l_\mathrm{W}( q(z \vert x), y )],
\end{align}
where the coefficients $\alpha_d$ are positive real numbers satisfying $\sum_{d=1}^N\alpha_d = 1$,
and the total label transition matrix from $\truelabels$ to $\weaklabels$ is defined as $T=(\alpha_1 T^{(1)\mathsf{T}}, \alpha_2 T^{(2)\mathsf{T}}, \dots, \alpha_N T^{(N)\mathsf{T}})\T$.
In fact, the $\alpha_d$ may be absorbed in a weak-label loss, and we may simply set $\alpha_d = 1/N$ for all $d$. The equality $T\T \oney = \onez$ can be verified by using $T^{(d)\mathsf{T}} \one_{\weaklabels^{(d)}} = \onez$ for all $d=1,2,\dots,N$,
and therefore, $T$ is formally qualified as a transition matrix.

As in the discussion in the main text, we assume that $T$ is left-invertible.
This assumption is weaker than requiring that all of the $T^{(d)}$ be left-invertible.
By using this $T$ as a label transition matrix, we can formally treat the multiple-source scenario exactly the same as the single-source case. In the training phase, we need to calculate the empirical risk. This can be done by first calculating the empirical partial risks from respective training sets with a partial-label set $\weaklabels^{(d)}$ and then aggregating the results.

\begin{example}
Consider three-class classification from two weakly labeled datasets. One set is labeled by an annotator who distinguishes Class~1 from the other classes, and the other set, by another annotator who distinguishes Class~2 from the other classes. Such a scenario is represented by the following transition matrices:
\begin{align}
    T^{(1)} &= \begin{pmatrix}
        1 & 0 & 0 \\
        0 & 1 & 1
    \end{pmatrix}, \quad
    T^{(2)} = \begin{pmatrix}
        0 & 1 & 0 \\
        1 & 0 & 1
    \end{pmatrix}, \quad
    T = \begin{pmatrix}
        \frac{1}{2} & 0 & 0 & \frac{1}{2} \\
        0 & \frac{1}{2} & \frac{1}{2} & 0 \\
        0 & \frac{1}{2} & 0 & \frac{1}{2}
    \end{pmatrix}\T.
\end{align}
Here, $T$ is reconstructible, while $T^{(1)}$ and $T^{(2)}$ are not. An example of $R$ is
\begin{align}
    R = \begin{pmatrix}
    1 & -1 & 1 & 1 \\
    1 & 1 & 1 & -1 \\
    -1 & 1 & -1 & 1 
    \end{pmatrix}.
\end{align}
\end{example}

\section{Proofs Omitted in Main Text}
\subsection{Theorem~\ref{thm:dual-rep}}\label{prf:dual-rep}

We first prove the following lemma, which relates Condition~1 of the theorem to the finiteness of the convex conjugate of $F(\vv)$.

\begin{lemma}\label{lemma:finite-S(q)}
Let $F: C\subset\onezperp\to\real$ be a closed convex function.
Then its convex conjugate $F^\ast(\qq)$ is finite for all $\qq\in\tldistr$
    if and only if
    $\sup_{\vv\in C}\left[ \max_{z\in\truelabels} v_z - F(\vv) \right]
    < \infty$.
\end{lemma}
\begin{proof}
Without loss of generality, the condition that $F^\ast(\qq)<\infty$ for all $\qq\in\tldistr$ can be replaced with the finiteness at $\qq=\vec{e}_y$ for all $y\in\weaklabels$,
where $\vec{e}_y \in \tldistr$ is a standard unit vector.
This is because of Jensen's inequality and the fact that $\tldistr$ is the convex hull of a set of the standard unit vectors.
Then, the lemma can be seen as a special case of Proposition~\ref{prop:equiv} with $\weaklabels=\truelabels$ and $R=I_\truelabels$.
\end{proof}

\begin{proof}[Proof of Theorem~\ref{thm:dual-rep}]

Suppose that $l$ is a regular proper loss. From Theorem~4, there exists a closed convex function ${S}: \tldistr\to\real$ and its subgradient function $\subgrad S: \tldistr\to\extreal^\truelabels$ such that
\begin{align}
    l(\qq,z) = - [ \subgrad S(\qq) ]_z
        + \inner{\qq}{\subgrad S(\qq)} - S(\qq).
\end{align}
Let $S^\ast: \Tilde{\domain} \to \real$ be the convex conjugate of $S$.
From the Savage representation and the identity $S^\ast(\subgrad S(\qq)) + S(\qq) = \inner{\qq}{\subgrad S(\qq)}$, which follows from the equality condition of the Fenchel-Young inequality, we have $l(\qq, z) = \lambda_{S^\ast}(\subgrad S(\qq), z)$.

Now we need to show that the restriction of $S^\ast$ to $\domain\equiv \Tilde{\domain} \cap \onezperp$, denoted as $F: \domain\to\real$, satisfies the two conditions of the theorem.
Suppose that $\vv = v_\parallel \onez + \vv_\perp$, where $v_\parallel\in\real$ and $\vv_\perp\in\onezperp$.
Because $\inner{\qq}{\vv} = v_\parallel + \inner{\qq}{\vv_\perp}$ for $\qq\in\tldistr$, it holds that
for all $\vv \in \tilde{\domain}$,
\begin{align}
    S^\ast(\vv)
    &\equiv \sup_{\qq\in\tldistr} \left[
        \inner{\qq}{\vv} - S(\qq) \right] \\
    &= v_\parallel + {F}(\vv_\perp).
\end{align}
This implies that
\begin{align}
\inner{\qq}{\vv} - S^\ast(\vv) = \inner{\qq}{\vv_\perp} - F(\vv_\perp)
\end{align}
for all $\qq\in\tldistr$ and $\vv\in\Tilde{\domain}$.
By taking the supremum of this equality over $\vv\in\tilde{\domain}$, we conclude that $F^\ast(\qq) = S(\qq)$ for all $\qq\in\tldistr$, where $F^\ast$ is the convex conjugate of $F$.
Because $S(\qq)$ is finite for all $\qq\in\tldistr$, $F^\ast(\qq)$ is also finite in $\tldistr$.
By Lemma~\ref{lemma:finite-S(q)}, this implies Condition~1 of the theorem.

To show Condition~2, we need to relate the subgradients of $S(\qq)$ with those of $F^\ast(\qq)$.
We first note that they have the same projections of the subgradients onto $\onezperp$, because $F^\ast(\qq)=S(\qq)$ in $\tldistr$.
Regarding the component of the subgradients that is parallel to $\onez$, it holds that $\inner{\onez}{\subgrad F^\ast(\qq)}=0$ because $F^\ast(\qq)$ is independent of $\inner{\qq}{\onez}$; that is, for all $t\in\real$,
\begin{align}
    F^\ast(\qq + t\onez)
    &= \sup_{\vv\in\domain} \left[
        \inner{\qq + t\onez}{\vv} - F(\vv)
    \right] \\
    &= \sup_{\vv\in\domain} \left[
        \inner{\qq}{\vv} - F(\vv)
    \right] \\
    &= F^\ast(\qq).
\end{align}
On the other hand, because $S$ is defined on $\tldistr$, it holds that $\vv + t \onez \in \partial S(\qq)$ for all $\vv\in \partial S(\qq)$ and $t\in\real$.
The choice of this $t$ does not affect the loss function's value.
Therefore, we can always choose $\subgrad S(\qq)$ such that $\subgrad S(\qq)=\subgrad F^\ast(\qq)$, which implies Condition~2.

Conversely, suppose that there exists a closed convex function $F:C\subset\onezperp\to\real$ that satisfies the two conditions.
Its convex conjugate $F^\ast$ is finite at all $\pp\in\tldistr$ by Lemma~\ref{lemma:finite-S(q)}.
Let $S: \tldistr\to\real$ be a restriction of $F^\ast$ on $\tldistr$.
In general, a subdifferential of a function is not larger as a set than a subdifferential of its restriction; that is, $\partial F^\ast(\qq) \subset \partial S(\qq)$.
This implies that $\subgrad F^\ast(\qq)$ can be seen as a subgradient function of $S$.
From this fact and the identity $F(\subgrad F^\ast(\qq)) + F^\ast(\qq) = \inner{\qq}{\subgrad F^\ast(\qq)}$, which follows from the equality condition of the Fenchel-Young inequality~\citep{rockafellar1996}, the loss $l(\qq,z) = \lambda_F(\subgrad F^\ast(\qq), z)$ conforms to the Savage representation and is proper.
\end{proof}

\subsection{Proposition~\ref{prop:min-in-domain}}\label{prf:min-in-domain}

Let $\vv_0$ be a minimizer of $\expect_{z\sim\pp} \left[
        \lambda_F(\vv, z)
    \right]$.
Then it holds that
\begin{align}
    - \inner{\vv_0}{\pp} + F(\vv_0)
    &= \min_{\vv\in\domain} \left[
        -\inner{\vv}{\pp} + F(\vv)
    \right] \\
    &= -F^\ast(\pp),
\end{align}
where $\vv_0 \in \partial F^\ast (\pp)$.
This proves the claim because $\pp$ is always a member of $\tldistr$.

\subsection{Theorem~\ref{thm:T-proper-dual-rep}}\label{prf:T-dual}

Our proof of this theorem relies on the following lemma that gives a general relation between $T$-properness and properness~\citep{cid-sueiro2012}.

\begin{lemma}\label{thm:equiv-loss}
A weak-label loss $l_\mathrm{W}: \tldistr\times\weaklabels \to \extreal$ is (strictly) \tproper if and only if a loss function $l: \tldistr\times\truelabels \to \extreal$, defined as $l(\qq, z) = \sum_{y\in\weaklabels} T_{yz} l_\mathrm{W}(\qq, y)$, is (strictly) proper.
\end{lemma}
This lemma follows from the identity $\expect_{y\sim T\pp} [l_\mathrm{W}(\qq,y)] = \expect_{z\sim \pp} [\sum_{y\in\weaklabels} T_{yz} l_\mathrm{W}(\qq,y)]$.
The left-hand (right-hand) side is minimized by $\qq=\pp$ if and only if the loss in the expected value is $T$-proper (proper).
The lemma indicates that having corrupted labels and a weak-label loss is equivalent to having clean labels and a mixed weak-label loss as a supervised-learning loss.

\begin{proof}[Proof of Theorem~\ref{thm:T-proper-dual-rep}]
From Lemma~\ref{thm:equiv-loss}, the weak-label loss $l_\mathrm{W}$ is $T$-proper if and only if a loss defined as ${l}(\qq, z) \equiv \sum_{y\in\weaklabels} T_{yz} {l}_\mathrm{W}(\qq, y)$ is proper. Then, from Theorem~5, there exists a closed convex function $F$ defined on a subset of $\onezperp$ that satisfies Condition~1 of the theorem and the following equation:
\begin{align}
    - [\subgrad F^\ast(\qq)]_z + F(\subgrad F^\ast(\qq))
    = \sum_{y\in\weaklabels} T_{yz} {l}_\mathrm{W}(\qq, y),
\end{align}
where $F^\ast$ is the convex conjugate of $F$ and $\subgrad F^\ast(\qq)$ is a subgradient of $F^\ast$ at a point $\qq$. By using the identity $T\T \oney = \onez$, we find that
\begin{align}
    - [\subgrad F^\ast(\qq)]_z = \sum_{y\in\weaklabels} T_{yz} [
    {l}_\mathrm{W}(\qq, y) - F(\subgrad F^\ast(\qq)) ].
\end{align}
Note that the right-hand side is a product of a matrix $T\T$ and a vector in $\real^\weaklabels$, whose $y$th component is $[{l}_\mathrm{W}(\qq, y) - F(\subgrad F^\ast(\qq)) ]$.
By the left-invertibility of $T$, we can invert this equation up to possibly nonzero $\vec{\Delta}(\qq)\in \coker T$.
\end{proof}

\subsection{Lemma~\ref{lemma:max-Rv-max-v}}\label{prf:maxRv}

Let $T$ be a label transition matrix corresponding to a reconstruction matrix $R$.
Then, $\vv = T\T (R\T \vv)$ for $\vv\in\onezperp$.
Because all the elements of $T$ are nonnegative and $T\T \oney = \onez$, a component of $\vv$ is a convex combination of $(R\T\vv)_y$.
Therefore, $v_z \leq \max_{y\in\weaklabels} ( R\T \vv )_y$ for all $z\in\truelabels$.

\subsection{Proposition~\ref{prop:equiv}}\label{prf:equiv}

Suppose that $F^\ast(R\vec{e}_y) < \infty$ for all $y\in\weaklabels$.
By the definition of the convex conjugate, it follows that
\begin{align}
    F^\ast(R\vec{e}_y)
    &\equiv \sup_{\vv\in\domain} [\inner{R\vec{e}_y}{\vv} - F(\vv)] \\
    &=  \sup_{\vv\in\domain} [\inner{\vec{e}_y}{R\T\vv} - F(\vv)] \\
    &=  \sup_{\vv\in\domain} [(R\T\vv)_y - F(\vv)] \\
    &< \infty.
\end{align}
By maximizing both sides over $y\in\weaklabels$, we obtain
\begin{align}
    \max_{y\in\weaklabels} F^\ast (R\vec{e}_y)
    &= \max_{y\in\weaklabels} \sup_{\vv\in\domain} \left[
        (R\T\vv)_y - F(\vv)
    \right] \\
    &= \sup_{\vv\in\domain} \left[
        \max_{y\in\weaklabels} (R\T\vv)_y - F(\vv)
    \right] \\
    &< \infty.
\end{align}

Conversely, suppose that
    $\sup_{\vv\in C} \left[ \max_{y\in\weaklabels} (R\T\vv)_y - F(\vv) \right]
    < \infty$.
Because $\inner{\qq}{R\T\vv} \leq \max_{y\in\weaklabels} (R\T\vv)_y$
    for all $\qq\in\wldistr$, it follows that
\begin{align}
    F^\ast(R\qq) &\equiv \sup_{\vv\in C} \left[
        \inner{R\qq}{\vv} - F(\vv)
    \right] \\
    &\equiv \sup_{\vv\in C} \left[
        \inner{\qq}{R\T\vv} - F(\vv)
    \right] \\
    &\leq \sup_{\vv\in C} \left[
        \max_{y\in\weaklabels} (R\T\vv)_y - F(\vv)
    \right] \\
    &< \infty.
\end{align}
The proposition follows as the special case with $\qq=\vec{e}_y$.

\section{Condition for Strict Properness}\label{apdx:strictly-proper}

In this appendix, we prove the conditions for a dual representation to give a strictly proper loss.

As already stated in Theorem~\ref{thm:savage}, a proper loss is strictly proper if and only if the associated negative Bayes risk $S(\pp)$ is strictly convex.
To dualize this condition, we introduce some notations.
Let $\truelabels'=\{ z'_1, z'_2, \dots , z'_{K'} \}$ be a subset of the label set $\truelabels=\{ z_1, z_2, \dots, z_K \}$.
Then, a mapping $\pi_{\truelabels'}: \real^\truelabels \to \real^{\truelabels'}$ denotes a natural projection from $\real^\truelabels$ onto $\real^{\truelabels'}$,
and $\rho_{\truelabels'}: \real^{\truelabels'} \to \one_{\truelabels'}^\perp$ denotes an orthogonal projection from $\real^{\truelabels'}$ onto its subspace $\one_{\truelabels'}^\perp$.
We also need a ``projection onto the bottom,'' $\sigma_{\truelabels'}: \one_{\truelabels'}^\perp \to \real^{\truelabels'-\{z'_1\}}$, which is a natural projection from $\one_{\truelabels'}^\perp$, as a subspace of $\real^{\truelabels'}$, onto $\real^{\truelabels'-\{z'_1\}}$. Let $F: \domain\to \real$ be a function whose domain is a convex subset $\domain$ of $\onezperp$.
Then, we define a function $F_{\truelabels'}: \domain_{\truelabels'} \to \real$ as the closure of
\begin{align}
    \tilde{F}_{\truelabels'}(\vv)
    \equiv 
    \inf_{\vv' \in \domain \cap \pi_{\truelabels'}^{-1} \circ \rho_{\truelabels'}^{-1}(\vv)}
    \left[
        F(\vv')
        - \inner{\frac{1}{\card{\truelabels'}}\one_{\truelabels'}}{\pi_{\truelabels'}(\vv')}
   \right],
\end{align}
where $\domain_{\truelabels'} \equiv \rho_{\truelabels'} \circ \pi_{\truelabels'} (\domain)$.

\begin{theorem}\label{thm:strictly-proper}
A proper loss in the dual representation associated with a closed convex function $F$ is strictly proper if and only if for all subsets $\truelabels'$ of $\truelabels$, a function $F_{\truelabels'} \circ \sigma^{-1}_{\truelabels'}$ is differentiable on some subset $\setD_{\truelabels'}$ of $\sigma_{\truelabels'} (\domain_{\truelabels'})$ and the range of $\partial F_{\truelabels'}$ on  $\sigma^{-1}_{\truelabels'}(\setD_{\truelabels'})$ contains the relative interior of $\distr(\truelabels')$.
\end{theorem}

\subsection{Proof}

Discussing strict convexity on a closed set $\tldistr$
    is not straightforward.
Instead, the following lemma allows us
    to decompose it into strict convexity on open subsets of $\tldistr$.
\begin{lemma}\label{lemma:strict-conv-diff-injectivity}
A function $S$ is strictly convex on $\tldistr$ if and only if
    it is strictly convex on a convex, relatively open subset
    $P_\truelabels(\truelabels')
        \equiv \{ \pp\in\tldistr \vert
        p_z \neq 0 \  (z\in\truelabels'), 
        p_z = 0 \  (z\notin\truelabels')\}$
    for any $\truelabels' \subset \truelabels$.
\end{lemma}
\begin{proof}
Suppose that $S$ is strictly convex on $\tldistr$.
Clearly, it is strictly convex on any convex subset of $\tldistr$.
Conversely, suppose that $S$ is strictly convex on
    $P_\truelabels(\truelabels')$
    for any subset $\truelabels'$ of $\truelabels$.
Let $\pp_1$ and $\pp_2$ be two different elements of $\tldistr$,
    let $l$ be a line segment connecting them, and
    let $\lambda$, $\lambda_1$, and $\lambda_2$ be real numbers such that
    $0< \lambda_2 < \lambda < \lambda_1 < 1$.
Also, define
    $\pp = \lambda \pp_1 + (1-\lambda) \pp_2$
    and $\pp_i' = \lambda_i \pp_1 + (1-\lambda_i) \pp_2$ ($i=1,2$).
We can see that $\pp$ is also a convex combination of $\pp_1'$ and $\pp_2'$:
\begin{align}
    \pp = \lambda' \pp_1' + (1-\lambda') \pp_2',
    \quad\text{where}\quad
    \lambda' = \frac{\lambda_2 - \lambda}{\lambda_1 - \lambda_2}.
\end{align}
Because the relative interior of $l$ is contained
    in only one of the $P_\truelabels(\truelabels')$,
    $\pp$, $\pp_1'$, and $\pp_2'$ are all
    contained in that $P_\truelabels(\truelabels')$.
Therefore, by assumption,
\begin{align}
    S(\pp)
    &< \lambda' S(\pp_1') + (1-\lambda') S(\pp_2') \\
    &\leq \lambda' [\lambda_1 S(\pp_1) + (1 - \lambda_1) S(\pp_2)]
    + (1-\lambda') [\lambda_2 S(\pp_1) + (1-\lambda_2) S(\pp_2)] \\
    &= \lambda S(\pp_1) + (1-\lambda) S(\pp_2).
\end{align}
As this holds true for any $\pp_1$ and $\pp_2$ in $\tldistr$
    and $\lambda \in (0,1)$,
    $S(\pp)$ is strictly convex on $\tldistr$.
\end{proof}

We now focus on a single $\truelabels'\subset\truelabels$ and $S$ restricted on $P_\truelabels(\truelabels')$.
Because $\inner{\pp}{\vv}=\inner{\pi_{\truelabels'}(\pp)}{\pi_{\truelabels'}(\vv)}$
    for any $\pp\in P_\truelabels(\truelabels')$
    and $\vv\in \domain$,
    we can define a function $S_{\truelabels'}$
    as
\begin{align}
    S_{\truelabels'}(\pp)
    \equiv \sup_{\vv\in C} [
        \inner{\pp}{\pi_{\truelabels'}(\vv)}
        - F(\vv)
    ]
\end{align}
for $\pp\in\pi_{\truelabels'}(P_\truelabels(\truelabels'))$,
and this function is equal to $S(\pp)$ in $\pi_{\truelabels'}(P_\truelabels(\truelabels'))$.
Clearly, $S_{\truelabels'}$ is strictly convex
    if and only if $S$ is strictly convex on $P_\truelabels(\truelabels')$.
Note that $\pi_{\truelabels'}(P_\truelabels(\truelabels'))$
    is the relative interior of $\distr(\truelabels')$,
    and that the above definition is applicable to points
    within the relative boundary of $P_\truelabels(\truelabels')$.
Therefore, $S_{\truelabels'}$ can be extended
    to a function on $\distr(\truelabels')$.

\begin{lemma}\label{lemma:strict-conv-subdiff-injectivity}
A convex function $f$ is strictly convex
    on a convex, relatively open subset $\domain$ of its domain
    if and only if
    $\partial f (\pp_1) \cap \partial f(\pp_2) =\emptyset$
    for any pair of two different points $\pp_1, \pp_2$
    in $\domain$.\footnote{The condition that $\domain$ is relatively open
    can be removed, in which case $\partial f(\pp)$ can be empty
    for some $\pp$.
    See, for example, Theorem~26.3 in \citet{rockafellar1996}.}
\end{lemma}
\begin{proof}
Suppose that $f$ is not strictly convex on $\domain$.
Then, there exist $\pp_1, \pp_2 \in \domain$ and
    $\lambda \in (0,1)$ such that
    $f(\lambda \pp_1 + (1-\lambda)\pp_2)
    = \lambda f(\pp_1) + (1-\lambda) f(\pp_2)$.
Take $\vv\in \partial f(\pp)$, and
    let $H$ be a graph of an affine function
    $h(\qq)\equiv f(\pp) + \inner{\qq - \pp}{\vv}$.
Then, $H$ is a supporting hyperplane of the epigraph of $f$ at $(\pp, f(\pp))$.
Because $(\pp, f(\pp))$ belongs to the relative interior of the line segment
    joining $(\pp_1, f(\pp_1))$ and $(\pp_2, f(\pp_2))$,
    these two points also lie in $H$.
Therefore, $\vv\in \partial f(\pp_1)$ and $\vv\in\partial f(\pp_2)$,
    which implies that $\partial f(\pp_1) \cap \partial f(\pp_2) \neq \emptyset$.

Conversely, suppose that there exist
    two different points $\pp_1$ and $\pp_2$ in $\domain$
    such that $\partial f(\pp_1) \cap \partial f(\pp_2) \neq \emptyset$.
Let $\vv$ be an element of $\partial f(\pp_1) \cap \partial f(\pp_2)$.
Then for a certain constant $k$, a graph $H$ of
    an affine function $h(\qq) = \inner{\qq-\pp}{\vv} + k$
    is a supporting hyperplane of the epigraph of $f$
    and contains $(\pp_1, f(\pp_1))$ and $(\pp_2, f(\pp_2))$.
This implies that $H$ contains the line segment joining
    $(\pp_1, f(\pp_1))$ and $(\pp_2, f(\pp_2))$.
Thus, $f$ cannot be strictly convex
    along the line segment connecting $\pp_1$ and $\pp_2$.
\end{proof}
By applying this lemma to $S_{\truelabels'}$,
    its strict convexity becomes equivalent to
    the injectivity of the subdifferential map $\partial S_{\truelabels'}$.
On the other hand,
    the inverse of a subdifferential map
    of a closed convex function is the subdifferential map
    of its conjugate function~(Corollary~23.5.1 in \citet{rockafellar1996}).
Indeed, it holds that
\begin{align}
    S_{\truelabels'}(\pp)
    &= \sup_{\vv\in C} \left[
        \inner{\pp}{\sigma_{\truelabels'}\circ\pi_{\truelabels'}(\vv)}
        + \inner{\frac{1}{\card{\truelabels'}} \one_{\truelabels'}}
            {\pi_{\truelabels'}(\vv)}
        - F(\vv)
    \right] \\
    &= \sup_{\vv'\in \sigma_{\truelabels'}\circ\pi_{\truelabels'}(C)}
        \left\{\sup_{\vv\in C\cap \pi_{\truelabels'}^{-1}\circ \sigma_{\truelabels'}^{-1}(\vv')} \left[
        \inner{\pp}{\vv'}
        + \inner{\frac{1}{\card{\truelabels'}} \one_{\truelabels'}}
            {\pi_{\truelabels'}(\vv)}
        - F(\vv)
    \right] \right\} \\
    &= \sup_{\vv'\in \sigma_{\truelabels'}\circ\pi_{\truelabels'}(C)} \left[
        \inner{\pp}{\vv'} - F_{\truelabels'}(\vv')
    \right],
\end{align}
and therefore, that
    $\partial S_{\truelabels'} = (\partial F_{\truelabels'})^{-1}$.
This implies that the necessary and sufficient condition for
    $\partial S_{\truelabels'}(\pp_1)
        \cap \partial S_{\truelabels'}(\pp_2) = \emptyset$
    for $\pp_1\neq \pp_2$
    is that
    $\{ \vv' \vert
        \{ \pp_1, \pp_2 \} \subset \partial F_{\truelabels'}(\vv') \}
    =\emptyset$.

\begin{lemma}
Suppose that $\pp \in \partial F_{\truelabels'}(\vv)$
    for some $\pp\in\interior\distr(\truelabels')$
    and $\vv$ in the domain of $\partial F_{\truelabels'}$.
Then, no other point in $\interior \distr(\truelabels')$ belongs
    to $\partial F_{\truelabels'}(\vv)$
    if and only if
    $F_{\truelabels'}\circ \sigma_{\truelabels'}^{-1}$
    is differentiable at $\sigma_{\truelabels'}(\vv)$.
\end{lemma}
\begin{proof}
Let $\vv_0$ be $\sigma_{\truelabels'}(\vv)$ and
    $\pp_0$ be $\sigma_{\truelabels'}\circ\rho_{\truelabels'}(\pp)$.
We can verify by direct calculation that
    $\partial (F_{\truelabels'}\circ \sigma_{\truelabels'}^{-1})(\vv_0)
    = \sigma_{\truelabels'} \circ \rho_{\truelabels'} (
        \partial F_{\truelabels'}(\vv)
    )$.
In addition, because $F_{\truelabels'}$ is defined on
    (a subset of) $\one_{\truelabels}^\perp$,
    $\qq + t\one_{\truelabels'} \in \partial F_{\truelabels'}(\vv)$
    for any $\qq \in \partial F_{\truelabels'}(\vv)$ and $t\in \real$.
Therefore, $\{\pp_0\} \in
    \partial (F_{\truelabels'}\circ \sigma_{\truelabels'}^{-1})(\vv_0)$
    is equivalent to
    $\partial F_{\truelabels'}(\vv) 
    = \{ \pp_0 + t\one_{\truelabels'} \vert t\in\real \}$,
    in which case $\partial F_{\truelabels'}(\vv)$ contains
    one and only one element of $\interior\distr(\truelabels')$.
This implies the lemma because
    a convex function $F_{\truelabels'}\circ \sigma_{\truelabels'}^{-1}$
    is differentiable at a point $\vv_0$ if and only if
    it has a unique subgradient there.
\end{proof}

Finally, by combining these lemmas, we obtain Theorem~\ref{thm:strictly-proper}.

\section{Forward-Correction Loss}\label{apdx:forwardcorrection}

In this appendix, we verify that a forward-corrected loss $l_\mathrm{W}$ conforms to Theorem~\ref{thm:T-proper-dual-rep}.
A weak-label loss $l_\mathrm{W}: \tldistr\times\weaklabels\to \extreal$ is called the forward correction of $l_\weaklabels$ if $l_\mathrm{W}(\qq, y) = l_\weaklabels(T\qq, y)$, where $l_\weaklabels: \wldistr\times\weaklabels \to \extreal$ is a proper loss for estimating weak-label posterior probabilities.

We first apply Theorem~\ref{thm:dual-rep} to $l_\weaklabels$ and find that $l_\weaklabels(\qq, y) = - [\subgrad F^\ast_\weaklabels(\qq)]_y + F_\weaklabels(\subgrad F^\ast_\weaklabels(\qq))$ for $\qq\in \wldistr$; here, $F^\ast_\weaklabels(\qq)$ is the negative Bayes risk corresponding to $l_\weaklabels$, and $F_\weaklabels(\vv)$ is its convex conjugate.
We also have the negative Bayes risk $S(\qq)$ and its conjugate $S^\ast(\vv)$ for the weak-label loss $l_\mathrm{W}$.
A key identity among these quantities is $S(\qq)=F^\ast_\weaklabels(T\qq)$, which further implies that $\subgrad S(\qq) = T\T \subgrad F^\ast_\weaklabels(T\qq)$.
The latter can be inverted to find that $\subgrad F^\ast_\weaklabels(T\qq) = R\T \subgrad S(\qq) - \vec{\Delta}(\qq)$ for some function $\vec{\Delta}(\qq)$ that takes values on $\coker T$.
It also holds that
\begin{align}
    F_\weaklabels(\subgrad F^\ast_\weaklabels(T\qq) ) + F^\ast_\weaklabels(T\qq)
    = \inner{T\qq}{\subgrad F^\ast_\weaklabels(T\qq)}\\
    = \inner{\qq}{\subgrad S(\qq)}
    = S^\ast(\subgrad S(\qq)) + S(\qq), \label{eq:rel-conj-funcs}
\end{align}
where the first and third equalities follow from the equality condition of the Fenchel-Young inequality~\citep{rockafellar1996}.
By using $S(\qq) = F^\ast_\weaklabels(T\qq)$ in Eq.~\eqref{eq:rel-conj-funcs}, we find that $F_\weaklabels(\subgrad F^\ast_\weaklabels(T\qq)) = S^\ast(\subgrad S(\qq))$.
Therefore, we confirm that
\begin{align}
    l_\mathrm{W}(\qq, y)
    &= l_\weaklabels(T\qq, y) \\
    &= - [R\T \subgrad S(\qq)]_y+ S^\ast(\subgrad S(\qq)) + \Delta_y(\qq),
\end{align}
which conforms toTheorem~\ref{thm:T-proper-dual-rep}.

\section{Linear-Algebraic Properties of Reconstruction Matrix}\label{apdx:R}

In this appendix, we present a proof that for any reconstructible label transition matrix $T$, there exists a reconstruction matrix $R$ such that $R\T \onez = \oney$.
Note that this further implies that $T\T (R\T \onez - \oney) = \vec{0}$, or that $R\T \onez - \oney \in \coker T$.

A transition matrix $T$ satisfies the identity $T\T \oney = \onez$.
This implies that for any $\vv \in \onezperp$,
\begin{align}
    \inner{T\vv}{\oney} = 0,
\end{align}
and thus, that $T\onezperp \subset \oney^\perp$.
Therefore, the restriction $T'$ of $T$ on $\onezperp$ has a left-inverse $R'$ defined on $\oney^\perp$.
A matrix $R=R' + k \onez \oney\T$ is also a left-inverse of $T'$.
Because $T\onez \notin \oney^\perp$, we can choose $k$ such that $R$ is a left-inverse of $T$.
For such $R$, it holds that $R\T\onez \propto \oney$, but because $T\T R\T \onez = \onez$ and $T\T \oney =\onez$, we conclude that $R\T \onez = \oney$.

\section{Experimental Details}\label{apdx:experiment}

\paragraph{Datasets}
We used the MNIST~\cite{lecun.etal1998} and CIFAR-10~\cite{krizhevsky2009} datasets.
Each dataset defines its training and test splits.
In our experiment, we split the training split into two splits: one was used for training, and the other was used for validation.
Table~\ref{tab:splits} lists the numbers of examples in the dataset splits.
\begin{table}
    \centering
    \caption{Numbers of examples in the dataset splits.
    ``Training (original)'' refers to those that are originally defined as training splits in the datasets,
    while ``Training (used)'' indicates those that were actually used in training.}
    \label{tab:splits}
    \vskip 0.1in
    \begin{small}
    \begin{tabular}{ccccc}
    \toprule
        Dataset & Training (original) & Training (used) & Validation & Test \\
    \midrule
        MNIST
            & 60,000
            & 54,000
            & 6,000
            & 10,000 \\
        CIFAR-10
            & 50,000
            & 45,000
            & 5,000
            & 10,000 \\
    \bottomrule
    \end{tabular}
    \end{small}
\end{table}

Before starting the experiments, we converted the ground-truth labels in the training splits into complementary labels.
A complementary label for an instance was randomly chosen with probabilities given by the transition matrix in Eq.~\eqref{eq:Tcomp}.

\paragraph{Training procedure}
\begin{table}
    \centering
    \caption{Initial learning rates with which the best validation accuracy was achieved for each setting.}
    \vskip 0.1in
    \begin{small}
    \begin{tabular}{cccccc}
        \toprule
        &
            Weight decay &
            MNIST, linear &
            MNIST, MLP &
            CIFAR-10, ResNet-20 &
            CIFAR-10, WRN-28-2 \\
        \midrule
        BC & 
            fixed &
            0.003  &
            0.0001 &
            0.001  &
            0.001  \\
        BC & 
            tuned &
            0.0001  &
            0.0001 &
            0.0003  &
            0.001  \\
        BC + GA &
            fixed &
            0.0001 &
            0.01   &
            0.01   &
            0.003  \\
        BC + GA &
            tuned &
            0.001 &
            0.01   &
            0.003   &
            0.003  \\
        BC + gLS &
            fixed &
            0.0003 &
            0.0003 &
            0.1    &
            0.03   \\
        \bottomrule
    \end{tabular}
    \end{small}
    \label{tab:learning_rate}
\end{table}
\begin{table}
    \centering
    \caption{Weight decay coefficient with which the best validation accuracy was achieved for BC and BC~+~GA.}
    \vskip 0.1in
    \begin{small}
    \begin{tabular}{ccccc}
        \toprule
        &
            MNIST, linear &
            MNIST, MLP &
            CIFAR-10, ResNet-20 &
            CIFAR-10, WRN-28-2 \\
        \midrule
        BC & 
            0.0003  &
            0.001 &
            0.001  &
            0.01  \\
        BC + GA &
            0.1 &
            0.001   &
            0.0003   &
            0.0003  \\
        \bottomrule
    \end{tabular}
    \end{small}
    \label{tab:weight_decay}
\end{table}
We used stochastic gradient descent with momentum to optimize the models.
The momentum and the mini-batch size were fixed to $0.9$ and $256$, respectively. The initial learning rates were chosen from $\{ 0.1, 0.03, 0.01, 0.003, 0.001, 0.0003, 0.0001\}$ as those giving the best validation accuracy.
When a learning rate of 0.1 or 0.0001 achieved the best validation accuracy, we also tried two more values beyond the predefined range.
In all such cases, we confirmed that the chosen values were at a peak or on a plateau of the validation accuracy.
The chosen values are listed in Table~\ref{tab:learning_rate}.
The default value of the weight decay coefficient is $10^{-4}$, but when it is tuned, it is chosen from $\{ 0.1, 0.03, 0.01, 0.003, 0.001, 0.0003, 0.0001\}$.
The values of the weight decay coefficient that achieved the best validation accuracy are listed in Table~\ref{tab:weight_decay}.

\begin{table}
    \centering
    \caption{Values of the coefficient $k$ in Eq.~\eqref{eq:squared-logits} as chosen by the validation accuracy.}
    \vskip 0.1in
    \begin{small}
    \begin{tabular}{cc}
    \toprule
        Dataset and model & $k$ \\
    \midrule
        MNIST, linear & 0.03 \\
        MNIST, MLP & 1.0 \\
        CIFAR-10, ResNet-20 & 1.0 \\
        CIFAR-10, WRN-28-2 & 1.0 \\
    \bottomrule
    \end{tabular}
    \end{small}
    \label{tab:k}
\end{table}
Our proposed method, generalized logit squeezing (gLS), has two hyperparameters: the exponent $\alpha$ and the coefficient $k$.
For a fixed $\alpha$, we searched for the value of $k$ that achieved the best validation accuracy.
The candidate values were 10, 3, 1, 0.3, 0.1, 0.03, and 0.01.
These results are listed in Table~\ref{tab:k}.

\begin{table}
    \centering
    \caption{Numbers of epochs at which the best validation accuracy was achieved for each setting.}
    \vskip 0.1in
    \begin{small}
    \begin{tabular}{ccccc}
        \toprule
        &
            MNIST, linear &
            MNIST, MLP &
            CIFAR-10, ResNet-20 &
            CIFAR-10, WRN-28-2 \\
        \midrule
        BC & 
            4.1 &
            46.6 &
            22.6 &
            24.6 \\
        BC + GA &
            54.1 &
            62.7 &
            83.6 &
            76.1 \\
        BC + gLS &
            19.5 &
            59.1 &
            47.9 &
            39.7 \\
        \bottomrule
    \end{tabular}
    \end{small}
    \label{tab:epochs}
\end{table}
We adopted early stopping to determine the training time.
Specifically, when the validation accuracy had not improved for the last 10 epochs, the learning rate was reduced by a factor of 10, and the third time the same condition was satisfied, the training was terminated.
The test accuracy reported here is for the epochs with the best validation accuracy.
Table~\ref{tab:epochs} lists the numbers of epochs at which the best validation accuracy was achieved.

We used a simple grid search strategy for the hyperparameter search.
The best hyperparameters (i.e., the learning rate and the gLS coefficient) were used in the evaluation step, in which a randomly initialized model was trained on the training split and evaluated on the test split.
The training duration in the evaluation step was also determined by the early stopping strategy as described above.

\paragraph{Other details}
All the experiments were performed using on-premise computation servers equipped with NVIDIA's GeForce GTX 1080Ti and Tesla V100.
The training duration varied significantly, depending on the methods and the model size, but the longest run took less than one hour on the Tesla V100.
We used PyTorch~\cite{pytorch} to implement the experiments.

\end{document}